\documentclass{article}

\usepackage{microtype}
\usepackage{graphicx}
\usepackage{booktabs} 

\usepackage{hyperref}

\usepackage[accepted]{icml2024}

\usepackage{amsmath}
\usepackage{amssymb}
\usepackage{mathtools}
\usepackage{amsthm}
\usepackage{enumitem}
\usepackage[capitalize,noabbrev]{cleveref}

\theoremstyle{plain}
\newtheorem{theorem}{Theorem}[section]

\newtheorem{lemma}[theorem]{Lemma}
\newtheorem{corollary}[theorem]{Corollary}
\theoremstyle{definition}
\newtheorem{definition}[theorem]{Definition}

\theoremstyle{remark}
\newtheorem{remark}[theorem]{Remark}

\usepackage[textsize=tiny]{todonotes}

\usepackage{multicol}
\usepackage{amsthm}
\usepackage{amsmath}
\DeclareMathOperator*{\argmax}{arg\,max}
\DeclareMathOperator*{\stochargmax}{stoch\,arg\,max}
\DeclareMathOperator*{\stochmax}{stoch\,max}
\usepackage{amssymb}
\usepackage{mathtools} 
\usepackage{booktabs} 
\usepackage{algorithm}
\usepackage{algorithmic}
\usepackage{hyperref}
\usepackage{graphicx}
\usepackage{subcaption}
\usepackage{float}
\usepackage{nameref}
\usepackage{zref-xr}
\DeclareMathOperator{\aset}{\mathcal{A}}
\DeclareMathOperator{\sset}{\mathcal{S}}
\DeclareMathOperator{\reals}{\mathbb{R}}
\DeclareMathOperator{\action}{\mathbf{a}}
\DeclareMathOperator{\state}{\mathbf{s}}
\newcommand{\expectx}[2]{\mathbb{E}_{#1} \left[ #2 \right]}

\icmltitlerunning{Stochastic Q-learning for Large Discrete Action Spaces}

\begin{document}

\twocolumn[
\icmltitle{Stochastic Q-learning for Large Discrete Action Spaces}




\begin{icmlauthorlist}
\icmlauthor{Fares Fourati}{yyy}
\icmlauthor{Vaneet Aggarwal}{comp}
\icmlauthor{Mohamed-Slim Alouini}{yyy}
\end{icmlauthorlist}

\icmlaffiliation{yyy}{Department of Computer, Electrical and Mathematical Science and Engineering, King Abdullah University of Science and Technology (KAUST), Thuwal, KSA.}
\icmlaffiliation{comp}{School of Industrial Engineering, Purdue University, West
Lafayette, IN 47907, USA}

\icmlcorrespondingauthor{Fares Fourati}{fares.fourati@kaust.edu.sa}

\icmlkeywords{Machine Learning, ICML}

\vskip 0.3in
]



\printAffiliationsAndNotice{}  

\begin{abstract}
In complex environments with large discrete action spaces, effective decision-making is critical in reinforcement learning (RL). Despite the widespread use of value-based RL approaches like Q-learning, they come with a computational burden, necessitating the maximization of a value function over all actions in each iteration. This burden becomes particularly challenging when addressing large-scale problems and using deep neural networks as function approximators.
In this paper, we present stochastic value-based RL approaches which, in each iteration, as opposed to optimizing over the entire set of $n$ actions, only consider a variable stochastic set of a sublinear number of actions, possibly as small as $\mathcal{O}(\log(n))$. The presented stochastic value-based RL methods include, among others, Stochastic Q-learning, StochDQN, and StochDDQN, all of which integrate this stochastic approach for both value-function updates and action selection. The theoretical convergence of Stochastic Q-learning is established, while an analysis of stochastic maximization is provided. Moreover, through empirical validation, we illustrate that the various proposed approaches outperform the baseline methods across diverse environments, including different control problems, achieving near-optimal average returns in significantly reduced time.
\end{abstract}

\section{Introduction}

Reinforcement learning (RL), a continually evolving field of machine learning, has achieved notable successes, especially when combined with deep learning \cite{sutton2018reinforcement, wang2022deep}. While there have been several advances in the field, a significant challenge lies in navigating complex environments with large discrete action spaces \citep{dulac2015deep, dulac2021challenges}. In such scenarios, standard RL algorithms suffer in terms of computational efficiency \cite{akkerman2023handling}. Identifying the optimal actions might entail cycling through all of them, in general, multiple times within different states, which is computationally expensive and may become prohibitive with large discrete action spaces \citep{tessler2019action}. 

Such challenges apply to various domains, including combinatorial optimization \citep{mazyavkina2021reinforcement, fourati2023randomized, fourati2024combinatorial, fourati2024federated}, natural language processing \citep{he2015deep, he-etal-2016-deep, he-etal-2016-deep-reinforcement, tessler2019action}, communications and networking \cite{luong2019applications, fourati2021artificial}, recommendation systems \citep{dulac2015deep},  transportation \cite{al2019deeppool,haliem2021distributed,li2022combining}, and robotics \citep{dulac2015deep, tavakoli2018action, tang2020discretizing, seyde2021bang, seyde2022solving, gonzalez2023asap, ireland2024revalued}. 
Although tailored solutions leveraging action space structures and dimensions may suffice in specific contexts, their applicability across diverse problems, possibly unstructured, still needs to be expanded. We complement these works by proposing a general method that addresses a broad spectrum of problems, accommodating structured and unstructured single and multi-dimensional large discrete action spaces.

Value-based and actor-based approaches are both prominent approaches in RL. Value-based approaches, which entail the agent implicitly optimizing its policy by maximizing a value function, demonstrate superior generalization capabilities but demand significant computational resources, particularly in complex settings. Conversely, actor-based approaches, which entail the agent directly optimizing its policy, offer computational efficiency but often encounter challenges in generalizing across multiple and unexplored actions \citep{dulac2015deep}. While both hold unique advantages and challenges, they represent distinct avenues for addressing the complexities of decision-making in large action spaces. However, comparing them falls outside the scope of this work. While some previous methods have focused on the latter \citep{dulac2015deep}, our work concentrates on the former. Specifically, we aim to exploit the natural generalization inherent in value-based RL approaches while reducing their per-step computational complexity.

Q-learning, as introduced by \citet{watkins1992q}, for discrete action and state spaces, stands out as one of the most famous examples of value-based RL methods and remains one of the most widely used ones in the field. As an off-policy learning method, it decouples the learning process from the agent's current policy, allowing it to leverage past experiences from various sources, which becomes advantageous in complex environments. In each step of Q-learning, the agent updates its action value estimates based on the observed reward and the estimated value of the best action in the next state. 

Some approaches have been proposed to apply Q-learning to continuous state spaces, leveraging deep neural networks \citep{mnih2013playing, van2016deep}. Moreover, several improvements have also been suggested to address its inherent estimation bias \citep{hasselt2010double, van2016deep, zhang2017weighted, lan2020maxmin, wang2021adaptive}. However, despite the different progress and its numerous advantages, a significant challenge still needs to be solved in Q-learning-like methods when confronted with large discrete action spaces. The computational complexity associated with selecting actions and updating Q-functions increases proportionally with the increasing number of actions, which renders the conventional approach impractical as the number of actions substantially increases. Consequently, we confront a crucial question: \textit{Is it possible to mitigate the complexity of the different Q-learning methods while maintaining a good performance?} 

This work proposes a novel, simple, and practical approach for handling general, possibly unstructured, single-dimensional or multi-dimensional, large discrete action spaces. Our approach targets the computational bottleneck in value-based methods caused by the search for a maximum ($\max$ and $\argmax$) in every learning iteration, which scales as $\mathcal{O}(n)$, i.e., linearly with the number of possible actions $n$. Through randomization, we can reduce this linear per-step computational complexity to logarithmic. 

We introduce $\stochmax$ and $\stochargmax$, which, instead of exhaustively searching for the precise maximum across the entire set of actions, rely on at most two random subsets of actions, both of sub-linear sizes, possibly each of size $\lceil\log(n)\rceil$. The first subset is randomly sampled from the complete set of actions, and the second from the previously exploited actions. These stochastic maximization techniques amortize the computational overhead of standard maximization operations in various Q-learning methods \citep{watkins1992q, hasselt2010double, mnih2013playing, van2016deep}. Stochastic maximization methods significantly accelerate the agent's steps, including action selection and value-function updates in value-based RL methods, making them practical for handling challenging, large-scale, real-world problems.

We propose Stochastic Q-learning, Stochastic Double Q-learning, StochDQN, and StochDDQN, which are obtained by changing $\max$ and $\argmax$ to $\stochmax$ and $\stochargmax$ in the Q-learning \citep{watkins1992q}, the Double Q-learning \cite{hasselt2010double}, the deep Q-network (DQN) \citep{mnih2013playing} and the double DQN (DDQN) \citep{van2016deep}, respectively. Furthermore, we observed that our approach works even for the on-policy Sarsa \cite{rummery1994line}.

We conduct a theoretical analysis of the proposed method, proving the convergence of Stochastic Q-learning, which integrates these techniques for action selection and value updates, and establishing a lower bound on the probability of sampling an optimal action from a random set of size $\lceil\log(n)\rceil$ and analyze the error of stochastic maximization compared to exact maximization. Furthermore, we evaluate the proposed RL algorithms on environments from Gymnasium \citep{brockman2016openai}. For the stochastic deep RL algorithms, the evaluations were performed on control tasks within the multi-joint dynamics with contact (MuJoCo) environment \citep{todorov2012mujoco} with discretized actions \citep{dulac2015deep, tavakoli2018action, tang2020discretizing}. These evaluations demonstrate that the stochastic approaches outperform non-stochastic ones regarding wall time speedup and sometimes rewards. Our key contributions are summarized as follows:
\begin{itemize}[leftmargin=*]
\item We introduce novel stochastic maximization techniques denoted as $\stochmax$ and $\stochargmax$, offering a compelling alternative to conventional deterministic maximization operations, particularly beneficial for handling large discrete action spaces, ensuring sub-linear complexity concerning the number of actions. 

\item We present a suite of value-based algorithms suitable for large discrete actions, including Stochastic Q-learning, Stochastic Sarsa, Stochastic Double Q-learning, StochDQN, and StochDDQN, which integrate stochastic maximization within Q-learning, Sarsa, Double Q-learning, DQN, and DDQN, respectively.

\item We analyze stochastic maximization and demonstrate the convergence of Stochastic Q-learning. Furthermore, we empirically validate our approach to tasks from the Gymnasium and MuJoCO environments, encompassing various dimensional discretized actions.
\end{itemize}

\section{Related Works}

While RL has shown promise in diverse domains, practical applications often grapple with real-world complexities. A significant hurdle arises when dealing with large discrete action spaces \citep{dulac2015deep, dulac2021challenges}. Previous research has investigated strategies to address this challenge by leveraging the combinatorial or the dimensional structures in the action space \cite{he-etal-2016-deep-reinforcement, tavakoli2018action, tessler2019action, delarue2020reinforcement, seyde2021bang, seyde2022solving, fourati2023randomized, fourati2024combinatorial, fourati2024federated, akkerman2023handling, fourati2024combinatorial, fourati2024federated, ireland2024revalued}. For example, \citet{he-etal-2016-deep-reinforcement} leveraged the combinatorial structure of their language problem through sub-action embeddings. Compressed sensing was employed in \citep{tessler2019action} for text-based games with combinatorial actions. \citet{delarue2020reinforcement} formulated the combinatorial action decision of a vehicle routing problem as a mixed-integer program.
Moreover, \citet{akkerman2023handling} introduced dynamic neighbourhood construction specifically for structured combinatorial large discrete action spaces. Previous works tailored solutions for multi-dimensional spaces such as those in \cite{seyde2021bang, seyde2022solving, ireland2024revalued}, among others, while practical in the multi-dimensional spaces, may not be helpful for single-dimensional large action spaces. While relying on the structure of the action space is practical in some settings, not all problems with large action spaces are multi-dimensional or structured. We complement these works by making no assumptions about the structure of the action space. 

Some approaches have proposed factorizing the action spaces to reduce their size. For example, these include factorizing into binary subspaces \citep{lagoudakis2003reinforcement, sallans2004reinforcement, pazis2011generalized, dulac2012fast}, expert demonstration \citep{tennenholtz2019natural}, tensor factorization \citep{mahajan2021reinforcement}, and symbolic representations \citep{cui2016online}. Additionally, some hierarchical and multi-agent RL approaches employed factorization as well \citep{zhang2020generating, kim2021learning, peng2021facmac, enders2023hybrid}. While some of these methods effectively handle large action spaces for certain problems, they necessitate the design of a representation for each discrete action. Even then, for some problems, the resulting space may still be large.

Methods presented in \citep{van2009using, dulac2015deep, wang2020bic} combine continuous-action policy gradients with nearest neighbour search to generate continuous actions and identify the nearest discrete actions. These are interesting methods but require continuous-to-discrete mapping and are mainly policy-based rather than value-based approaches.
In the works of \citet{kalashnikov2018qt} and \citet{quillen2018deep}, the cross-entropy method \citep{rubinstein1999cross} was utilized to approximate action maximization. This approach requires multiple iterations ($r$) for a single action selection. During each iteration, it samples $n'$ values, where $n' < n$, fits a Gaussian distribution to $m < n'$ of these samples, and subsequently draws a new batch of $n'$ samples from this Gaussian distribution. As a result, this approximation remains costly, with a complexity of $\mathcal{O}(rn')$. Additionally, in the work of \citet{van2020q}, a neural network was trained to predict the optimal action in combination with a uniform search. This approach involves the use of an expensive autoregressive proposal distribution to generate $n'$ actions and samples a large number of actions ($m$), thus remaining computationally expensive, with $\mathcal{O}(n'+m)$. In \citep{metz2017discrete}, sequential DQN allows the agent to choose sub-actions one by one, which increases the number of steps needed to solve a problem and requires $d$ steps with a linear complexity of $\mathcal{O}(i)$ for a discretization granularity $i$. Additionally, \citet{tavakoli2018action} employs a branching technique with duelling DQN for combinatorial control problems. Their approach has a complexity of $\mathcal{O}(di)$ for actions with discretization granularity $i$ and $d$ dimensions, whereas our method, in a similar setting, achieves $\mathcal{O}(d\log(i))$. Another line of work introduces action elimination techniques, such as the action elimination DQN \citep{zahavy2018learn}, which employs an action elimination network guided by an external elimination signal from the environment. However, it requires this domain-specific signal and can be computationally expensive ($\mathcal{O}(n')$ where $n' \leq n$ are the number of remaining actions). In contrast, curriculum learning, as proposed by \citet{farquhar2020growing}, initially limits an agent's action space, gradually expanding it during training for efficient exploration. However, its effectiveness relies on having an informative restricted action space, and as the action space size grows, its complexity scales linearly with its size, eventually reaching $\mathcal{O}(n)$.

In the context of combinatorial bandits with a single state but large discrete action spaces, previous works have exploited the combinatorial structure of actions, where each action is a subset of main arms. For instance, for submodular reward functions, which imply diminishing returns when adding arms, in \cite{fourati2023randomized} and \cite{fourati2024combinatorial}, stochastic greedy algorithms are used to avoid exact search. The former evaluates the marginal gains of adding and removing sub-actions (arms), while the latter assumes monotonic rewards and considers adding the best arm until a cardinality constraint is met. For general reward functions, \citet{fourati2024federated} propose using approximation algorithms to evaluate and add sub-actions. While these methods are practical for bandits, they exploit the combinatorial structure of their problems and consider a single-state scenario, which is different from general RL problems.

While some approaches above are practical for handling specific problems with large discrete action spaces, they often exploit the dimensional or combinatorial structures inherent in their considered problems. In contrast, we complement these approaches by proposing a solution to tackle any general, potentially unstructured, single-dimensional or multi-dimensional, large discrete action space without relying on structure assumptions. Our proposed solution is general, simple, and efficient.

\section{Problem Description}
\label{sec:probdesc}

In the context of a Markov decision process (MDP), we have specific components: a finite set of actions denoted as $\aset$, a finite set of states denoted as $\sset$, a transition probability distribution $\mathcal{P}: \sset \times \aset \times \sset \rightarrow [0,1]$, a bounded reward function $r: \sset \times \aset \rightarrow \reals$, and a discount factor $\gamma \in [0,1]$.
Furthermore, for time step $t$, we denote the chosen action as $\action_t$, the current state as $\state_t$, and the received reward as $r_t \triangleq r(\state_t, \action_t)$. Additionally, for time step $t$, we define a learning rate function $\alpha_t: \sset \times \aset \rightarrow [0,1]$.

The cumulative reward an agent receives during an episode in an MDP with variable length time $T$ is the return $R_t$. It is calculated as the discounted sum of rewards from time step $t$ until the episode terminates: $R_t \triangleq \sum_{i=t}^T \gamma^{i-t} r_i$. RL aims to learn a policy $\pi: \sset \rightarrow \aset$ mapping states to actions that maximize the expected return across all episodes.
The state-action value function, denoted as $Q^{\pi}(\state, \action)$, represents the expected return when starting from a given state $\state$, taking action $\action$, and following a policy $\pi$ afterwards. The function $Q^\pi$ can be expressed recursively using the Bellman equation:
\begin{equation}
Q^\pi(\state, \action) = r(\state ,\action) + \gamma \sum_{\state'\in \sset} \mathcal{P}(\state'|\state,\action) Q^\pi(\state',\pi(\state')).
\end{equation}

Two main categories of policies are commonly employed in RL systems: value-based and actor-based policies \citep{sutton2018reinforcement}. This study primarily concentrates on the former type, where the value function directly influences the policy's decisions. An example of a value-based policy in a state $\state$ involves an $\varepsilon_{\state}$-greedy algorithm, selecting the action with the highest Q-function value with probability $(1-\varepsilon_{\state})$, where $\varepsilon_{\state} \geq 0$, function of the state $\state$, requiring the use of $\argmax$ operation, as follows:
\begin{equation}
\label{eq:argmaxPolicy}
\pi_Q(\state) =
\begin{cases}
\text{play randomly} & \text{with proba. } \epsilon_{\state} \\
\argmax_{\action \in \aset} Q(\state,\action) & \text{otherwise.} 
\end{cases}
\end{equation}
Furthermore, during the training, to update the Q-function, Q-learning \citep{watkins1992q}, for example, uses the following update rule, which requires a $\max$ operation: 
\begin{align}
\label{eq:update}
&Q_{t+1}\left(\state_t, \action_t\right)=\left(1-\alpha_t\left(\state_t, \action_t\right)\right) Q_t\left(\state_t, \action_t\right) \nonumber \\
&+\alpha_t\left(\state_t, \action_t\right)\left[r_t+\gamma \max _{b \in \mathcal{A}} Q_t\left(\state_{t+1}, b\right)\right] .  
\end{align}
Therefore, the computational complexity of both the action selections in Eq. (\ref{eq:argmaxPolicy}) and the Q-function updates in Eq. (\ref{eq:update}) scales linearly with the cardinality $n$ of the action set $\aset$, making this approach infeasible as the number of actions increases significantly. The same complexity issues remain for other Q-learning variants, such as Double Q-learning \citep{hasselt2010double}, DQN \citep{mnih2013playing}, and DDQN \citep{van2016deep}, among several others.

When representing the value function as a parameterized function, such as a neural network, taking only the current state $\state$ as input and outputting the values for all actions, as proposed in DQN \cite{mnih2013playing}, the network must accommodate a large number of output nodes, which results in increasing memory overhead and necessitates extensive predictions and maximization over these final outputs in the last layer. A notable point about this approach is that it does not exploit contextual information (representation) of actions, if available, which leads to lower generalization capability across actions with similar features and fails to generalize over new actions.

Previous works have considered generalization over actions by taking the features of an action $\action$ and the current state $\state$ as inputs to the Q-network and predicting its value \cite{zahavy2018learn, metz2017discrete, van2020q}. However, it leads to further complications when the value function is modeled as a parameterized function with both state $\state$ and action $\action$ as inputs. Although this approach allows for improved generalization across the action space by leveraging contextual information from each action and generalizing across similar ones, it requires evaluating the function for each action within the action set $\aset$. This results in a linear increase in the number of function calls as the number of actions grows. This scalability issue becomes particularly problematic when dealing with computationally expensive function approximators, such as deep neural networks \citep{dulac2015deep}. Addressing these challenges forms the motivation behind this work.

\section{Proposed Approach}

To alleviate the computational burden associated with maximizing a Q-function at each time step, especially when dealing with large action spaces, we introduce stochastic maximization methods with sub-linear complexity relative to the size of the action set $\aset$. Then, we integrate these methods into different value-based RL algorithms.

\subsection{Stochastic Maximization}
\label{stochmax_section}

We introduce stochastic maximization as an alternative to maximization when dealing with large discrete action spaces. Instead of conducting an exhaustive search for the precise maximum across the entire set of actions $\mathcal{A}$, stochastic maximization searches for a maximum within a stochastic subset of actions of sub-linear size relative to the total number of actions. In principle, any size can be used, trading off time complexity and approximation. We mainly focus on $\mathcal{O}(\log(n))$ to illustrate the power of the method in recovering Q-learning, even with such a small number of actions, with logarithmic complexity.

We consider two approaches to stochastic maximization: memoryless and memory-based approaches. The memoryless one samples a random subset of actions $\mathcal{R} \subseteq \aset$ with a sublinear size and seeks the maximum within this subset. On the other hand, the memory-based one expands the randomly sampled set to include a few actions $\mathcal{M}$ with a sublinear size from the latest exploited actions $\mathcal{E}$ and uses the combined sets to search for a stochastic maximum. Stochastic maximization, which may miss the exact maximum in both versions, is always upper-bounded by deterministic maximization, which finds the exact maximum. However, by construction, it has sublinear complexity in the number of actions, making it appealing when maximizing over large action spaces becomes impractical.

Formally, given a state $\state$, which may be discrete or continuous, along with a Q-function, a random subset of actions $\mathcal{R}\subseteq \aset$, and a memory subset $\mathcal{M}\subseteq \mathcal{E}$ (empty in the memoryless case), each subset being of sublinear size, such as at most $\mathcal{O}(\log(n))$ each, the $\stochmax$ is the maximum value computed from the union set $\mathcal{C} = \mathcal{R} \cup \mathcal{M}$, defined as:
\begin{equation}
\label{stoch_max_def}
\stochmax_{k\in \aset} Q_t(\state,k) \triangleq \max_{k\in \mathcal{C}} Q_t(\state,k).
\end{equation}
Besides, the $\stochargmax$ is computed as follows:
\begin{equation}
\label{stoch_arg_max_def}
\stochargmax_{k\in \aset} Q_t(\state,k) \triangleq \argmax_{k\in \mathcal{C}} Q_t(\state,k).
\end{equation}

In the analysis of stochastic maximization, we explore both memory-based and memoryless maximization. In the analysis and experiments, we consider the random set $\mathcal{R}$ to consist of $\lceil\log(n)\rceil$ actions. When memory-based, in our experiments, within a given discrete state, we consider the two most recently exploited actions in that state. For continuous states, where it is impossible to retain the latest exploited actions for each state, we consider a randomly sampled subset $\mathcal{M} \subseteq \mathcal{E}$, which includes $\lceil\log(n)\rceil$ actions, even though they were played in different states. We demonstrate that this approach was sufficient to achieve good results in the benchmarks considered; see Section \ref{Stochastic_DQN_Reward_Analysis}. Our Stochastic Q-learning convergence analysis considers memoryless stochastic maximization with a random set $\mathcal{R}$ of any size.

\begin{remark}
By setting $\mathcal{C}$ equal to $\mathcal{A}$, we essentially revert to standard approaches. Consequently, our method is an extension of non-stochastic maximization. However, in pursuit of our objective to make RL practical for large discrete action spaces, for a given state $\state$, in our analysis and experiments, we keep the union set $\mathcal{C}$ limited to at most $2\lceil\log(n)\rceil$, ensuring sub-linear (logarithmic) complexity.
\end{remark}

\subsection{Stochastic Q-learning}

\begin{algorithm}[t]
\begin{algorithmic}
\small
\caption{Stochastic Q-learning}
\label{alg:stoc_q_learning}
\STATE Initialize $Q(\state, \action)$ for all $\state \in \mathcal{S}, \action \in \mathcal{A}$
\FOR{each episode}
    \STATE Observe state $\state$.
    \FOR{each step of episode}
        \STATE Choose $\action$ from $\state$ with policy $\pi_Q^S(\state)$.
        \STATE Take action $\action$, observe $r$, $\state'$.
        \STATE $b^* \leftarrow \stochargmax_{b\in \mathcal{A}} Q(\state',b).$
        \STATE $\Delta \leftarrow r + \gamma  Q(\state', b^*) - Q(\state, \action)$.
        \STATE $Q(\state, \action) \leftarrow Q(\state, \action) + \alpha(\state,\action) \Delta$.\\
        \STATE $\state \leftarrow \state'$.
    \ENDFOR
\ENDFOR
\end{algorithmic}
\end{algorithm}

We introduce Stochastic Q-learning, described in Algorithm \ref{alg:stoc_q_learning}, and Stochastic Double Q-learning, described in Algorithm \ref{alg:stoc_double_q_learning} in Appendix \ref{appendix_pseudo_codes}, that replace the $\max$ and $\argmax$ operations in Q-learning and Double Q-learning with $\stochmax$ and $\stochargmax$, respectively. Furthermore, we introduce Stochastic Sarsa, described in Algorithm \ref{alg:sarsa} in Appendix \ref{appendix_pseudo_codes}, which replaces the maximization in the greedy action selection ($\argmax$) in Sarsa. 

Our proposed solution takes a distinct approach from the conventional method of selecting the action with the highest Q-function value from the complete set of actions $\aset$. Instead, it uses stochastic maximization, which finds a maximum within a stochastic subset $\mathcal{C}\subseteq\mathcal{A}$, constructed as explained in Section \ref{stochmax_section}. Our stochastic policy $\pi_Q^S(\state)$, uses an $\varepsilon_{\state}$-greedy algorithm, in a given state $\state$, with a probability of $(1-\varepsilon_{\state})$, for $\varepsilon_{\state} > 0$, is defined as follows:
\begin{equation}
\label{eq:stochargmaxPolicy}
\pi_Q^S(\state) \triangleq
\begin{cases}
\text{play randomly} & \text{with proba. } \epsilon_{\state} \\
\stochargmax_{\action \in \aset} Q(\state,\action) & \text{otherwise}. 
\end{cases}
\end{equation}
Furthermore, during the training, to update the Q-function, our proposed Stochastic Q-learning uses the following rule: 
\begin{align}
\label{eq:stoch_q_updates}
&Q_{t+1}\left(\state_t, \action_t\right)=\left(1-\alpha_t\left(\state_t, \action_t\right)\right) Q_t\left(\state_t, \action_t\right) \nonumber \\
&+\alpha_t\left(\state_t, \action_t\right)\left[r_t+ \gamma\stochmax _{b \in \mathcal{A}} Q_t\left(\state_{t+1}, b\right)\right] .    
\end{align}
While Stochastic Q-learning, like Q-learning, employs the same values for action selection and action evaluation, Stochastic Double Q-learning, similar to Double Q-learning, learns two separate Q-functions. For each update, one Q-function determines the policy, while the other determines the value of that policy. Both stochastic learning methods remove the maximization bottleneck from exploration and training updates, making these proposed algorithms significantly faster than their deterministic counterparts.

\subsection{Stochastic Deep Q-network}

We introduce Stochastic DQN (StochDQN), described in Algorithm \ref{alg:StochDQN} in Appendix \ref{appendix_pseudo_codes}, and Stochastic DDQN (StochDDQN) as efficient variants of deep Q-networks. These variants substitute the maximization steps in the DQN \citep{mnih2013playing} and DDQN \citep{van2016deep} algorithms with the stochastic maximization operations. In these modified approaches, we replace the $\varepsilon_{\state}$-greedy exploration strategy with the same exploration policy as in Eq. \eqref{eq:stochargmaxPolicy}.

For StochDQN, we employ a deep neural network as a function approximator to estimate the action-value function, represented as $Q(\state,\action;\theta) \approx Q(\state,\action)$, where $\theta$ denotes the weights of the Q-network. This network is trained by minimizing a series of loss functions denoted as $L_i(\theta_i)$, with these loss functions changing at each iteration $i$ as follows:
\begin{align}
\label{eq:q-learning-loss}
L_i\left(\theta_i\right) &\triangleq \expectx{\state,\action \sim \rho(\cdot)}{\left(y_i - Q \left(\state,\action ; \theta_i \right) \right)^2},
\end{align}
where $y_i \triangleq \expectx{}{r + \gamma \stochmax_{b \in \aset} Q(\state', b; \theta_{i-1}) | \state, \action }$.
In this context, $y_i$ represents the target value for an iteration $i$, and $\rho(.)$ is a probability distribution that covers states and actions. Like the DQN approach, we keep the parameters fixed from the previous iteration, denoted as $\theta_{i-1}$ when optimizing the loss function $L_i(\theta_i)$. 

These target values depend on the network weights, which differ from the fixed targets typically used in supervised learning. We employ stochastic gradient descent for the training. While StochDQN, like DQN, employs the same values for action selection and evaluation, StochDDQN, like DDQN, trains two separate value functions. It does this by randomly assigning each experience to update one of the two value functions, resulting in two sets of weights, $\theta$ and $\theta'$. For each update, one set of weights determines the policy, while the other set determines the values.

\section{Stochastic Maximization Analysis}

In the following, we study stochastic maximization with and without memory compared to exact maximization.

\subsection{Memoryless Stochastic Maximization}

Memoryless stochastic maximization, i.e., $\mathcal{C}= \mathcal{R} \cup\emptyset$, does not always yield an optimal maximizer. To return an optimal action, this action needs to be randomly sampled from the set of actions. Finding an exact maximizer, without relying on memory $\mathcal{M}$, is a random event with a probability $p$, representing the likelihood of sampling such an exact maximizer. In the following lemma, we provide a lower bound on the probability of discovering an optimal action within a uniformly randomly sampled subset $\mathcal{C}= \mathcal{R}$ of $\lceil\log(n)\rceil$ actions, which we prove in Appendix \ref{prooflemma}.

\begin{lemma}
\label{proba_lemma}
For any given state $\state$, the probability $p$ of sampling an optimal action from a uniformly randomly chosen subset $\mathcal{C}$ of size $\lceil\log(n)\rceil$ actions is at least $\frac{\lceil\log(n)\rceil}{n}$. 
\end{lemma}

While finding an exact maximizer through sampling may not always occur, the rewards of near-optimal actions can still be similar to those obtained from an optimal action. Therefore, the difference between stochastic maximization and exact maximization might be a more informative metric than just the probability of finding an exact maximizer. Thus, at time step $t$, given state $\state$ and the current estimated Q-function $Q_t$, we define the estimation error as $\beta_t(\state)$, as follows:
\begin{equation}
\label{beta_t_def}
\beta_t(\state) \triangleq \max _{\action \in \mathcal{A}} Q_t\left(\state, \action\right) - \stochmax _{\action \in \mathcal{A}} Q_t\left(\state, \action\right).
\end{equation}
Furthermore, we define the similarity ratio $\omega_t(\state)$, as follows:
\begin{equation}
\label{omega_t_def}
\omega_t(\state) \triangleq \stochmax _{\action \in \mathcal{A}} Q_t\left(\state, \action\right) / \max _{\action \in \mathcal{A}} Q_t\left(\state, \action\right).
\end{equation}
It can be seen from the definitions that $\beta_t(\state) \geq 0$ and $\omega_t(\state) \leq 1$. While sampling the exact maximizer is not always possible, near-optimal actions may yield near-optimal values, providing good approximations, i.e., $\beta_t(\state) \approx 0$ and $\omega_t(\state) \approx 1$. In general, this difference depends on the value distribution over the actions. %

While we do not make any specific assumptions about the value distribution in our work, we note that with some simplifying assumptions on the value distributions over the actions, one can derive more specialized guarantees. For example, assuming that the rewards are uniformly distributed over the actions, we demonstrate in \cref{stoch_max_with_uniform} that for a given discrete state $\state$, if the values of the sampled actions independently follow a uniform distribution from the interval $[Q_t(\state,\action_t^{\star})-b_t(\state), Q_t(\state,\action_t^{\star})]$, where $b_t(\state)$ represents the range of the $Q_t(\state,.)$ values over the actions in state $\state$ at time step $t$, then the expected value of $\beta_t(\state)$, even without memory, is:
$ \mathbb{E}\left[\beta_t(\state) \mid \state \right] \leq \frac{b_t(\state)}{\lceil\log(n)\rceil+1}.
$
Furthermore, we empirically demonstrate that for the considered control problems, the difference $\beta_t(\state)$ is not large, and the ratio $\omega_t(\state)$ is close to one, as shown in Section \ref{stochmax_experiments}.

\subsection{Stochastic Maximization with Memory}

While memoryless stochastic maximization could approach the maximum value or find it with the probability $p$, lower-bounded in Lemma \ref{proba_lemma}, it does not converge to an exact maximization, as it keeps sampling purely at random, as can be seen in Fig. \ref{reward:uniform} in Appendix \ref{uniform_appendx_exp}. However, memory-based stochastic maximization, i.e., $\mathcal{C}=\mathcal{R}\cup\mathcal{M}$ with $\mathcal{M} \neq \emptyset$, can become an exact maximization when the Q-function becomes stable, as we state in the Corollary \ref{cor_stable}, which we prove in Appendix \ref{corollary_proof}, and as confirmed in Fig. \ref{reward:uniform}. 
\begin{definition}
A Q-function is considered stable for a given time range and state $\state$ when its maximizing action in that state remains unchanged for all subsequent steps within that time, even if the Q-function's values themselves change.
\end{definition}
A straightforward example of a stable Q-function occurs during validation periods when no function updates are performed. However, in general, a stable Q-function does not have to be static and might still vary over the rounds; the critical characteristic is that its maximizing action remains the same even when its values are updated. Although the $\stochmax$ has sub-linear complexity compared to the $\max$, without any assumption of the value distributions, the following corollary shows that, on average, for a stable Q-function, after a certain number of iterations, the output of the $\stochmax$ matches precisely the output of $\max$.

\begin{corollary}
\label{cor_stable}
For a given state $\state$, assuming a time range where the Q-function becomes stable in that state, $\beta_t(\state)$ is expected to converge to zero after $\frac{n}{\lceil\log(n)\rceil}$ iterations.
\end{corollary}

Recalling the definition of the similarity ratio $\omega_t$, it follows that $\omega_t(\state) = 1 - \beta(s)/\max_{\action \in \aset}Q_t(\state,\action)$. Therefore, for a given state $\state$, where the Q-function becomes stable, given the boundedness of iterates in Q-learning, it is expected that $\omega_t$ converges to one. This observation was confirmed, even with continuous states and using neural networks as function approximators, in Section \ref{stochmax_experiments}.

\section{Stochastic Q-learning Convergence}

In this section, we analyze the convergence of the Stochastic Q-learning, described in Algorithm \ref{alg:stoc_q_learning}. This algorithm employs the policy $\pi_Q^S(\state)$, as defined in Eq.~(\ref{eq:stochargmaxPolicy}), with $\varepsilon_{\state}>0$ to guarantee that $\mathbb{P}_{\pi}[\action_t = \action \mid \state_t = \state] > 0$ for all state-action pairs $(\state,\action) \in \sset \times \aset$. The value update rule, on the other hand, uses the update rule specified in Eq. (\ref{eq:stoch_q_updates}). 

In the convergence analysis, we focus on memoryless maximization. While the $\stochargmax$ operator for action selection can be employed with or without memory, we assume a memoryless $\stochmax$ operator for value updates, which means that value updates are performed by maximizing over a randomly sampled subset of actions from $\mathcal{A}$, sampled independently from both the next state $\state'$ and the set used for the $\stochargmax$.

For a stochastic variable subset of actions $\mathcal{C} \subseteq \mathcal{A}$, following some probability distribution $\mathbb{P}: 2^\mathcal{A} \rightarrow [0,1]$, we consider, without loss of generality $Q(., \emptyset)=0$, and define, according to $\mathbb{P}$, a target Q-function, denoted as $Q^*$, as:
\begin{equation}
\label{qstar}
Q^*(\state, \action) \triangleq \mathbb{E}\left[r(\state, \action)+\gamma \max _{b \in \mathcal{C} \sim \mathbb{P}} Q^*(\state', b) \mid \state, \action \right].
\end{equation}

\begin{remark}
The $Q^*$ defined above depends on the sampling distribution $\mathbb{P}$. Therefore, it does not represent the optimal value function of the original MDP problem; instead, it is optimal under the condition where only a random subset of actions following the distribution $\mathbb{P}$ is available to the agent at each time step. However, as the sampling cardinality increases, it increasingly better approximates the optimal value function of the original MDP and fully recovers the optimal Q-function of the original problem when the sampling distribution becomes $\mathbb{P}(\mathcal{A})=1$.
\end{remark}

The following theorem states the convergence of the iterates $Q_t$ of Stochastic Q-learning with memoryless stochastic maximization to the $Q^*$, defined in Eq. \ref{qstar}, for any sampling distribution $\mathbb{P}$, regardless of the cardinality.

\begin{theorem}
\label{theorem}
For a finite MDP, as described in Section \ref{sec:probdesc}, let $\mathcal{C}$ be a randomly independently sampled subset of actions from $\mathcal{A}$, of any cardinality, following any distribution $\mathbb{P}$, exclusively sampled for the value updates, for the Stochastic Q-learning, as described in Algorithm \ref{alg:stoc_q_learning}, given by the following update rule:
\begin{align}
&Q_{t+1}\left(\state_t, \action_t\right) = (1-\alpha_t\left(\state_t, \action_t\right))Q_t\left(\state_t, \action_t\right) \nonumber \\
&+\alpha_t\left(\state_t, \action_t\right)\left[r_t+\gamma \max _{b \in \mathcal{C} \sim \mathbb{P}} Q_t\left(\state_{t+1}, \action\right)\right], \nonumber
\end{align}
given any initial estimate $Q_0$, $Q_{t}$ converges with probability 1 to $Q^*$, defined in Eq. (\ref{qstar}), as long as
$
\sum_t \alpha_t(\state, \action)=\infty$ and $\sum_t \alpha_t^2(\state, \action)<\infty 
$
for all $(\state,\action) \in \sset \times \aset$.
\end{theorem}

The theorem's result demonstrates that for any cardinality of actions, Stochastic Q-learning converges to $Q^*$, as defined in Eq. (\ref{qstar}), which recovers the convergence guarantees of Q-learning when the sampling distribution is $\mathbb{P}(\mathcal{A})=1$.

\begin{remark}
In principle, any size can be used, balancing time complexity and approximation. Our empirical experiments focused on $\log(n)$ to illustrate the method's ability to recover Q-learning, even with a few actions. Using $\sqrt{n}$ will approach the value function of Q-learning more closely compared to using $\log(n)$, albeit at the cost of higher complexity than $\log(n)$.    
\end{remark}

The theorem shows that even with memoryless stochastic maximization, using randomly sampled $\mathcal{O}(\log(n))$ actions, the convergence is still guaranteed. However, relying on memory-based stochastic maximization helps minimize the approximation error in stochastic maximization, as shown in \cref{cor_stable}, and outperforms Q-learning as shown in the experiments in Section \ref{Stochastic_Q_learning_Reward_Analysis}. 

In the following, we provide a sketch of the proof addressing the extra stochasticity due to stochastic maximization. The full proof is provided in Appendix \ref{appendix_theorem_proof}. 

We tackle the additional stochasticity depending on the sampling distribution $\mathbb{P}$, by defining an operator function $\Phi$, which for any $q: \sset \times \aset \rightarrow \reals$, is as follows:
\begin{align}
\label{Phi_operator}
(\Phi q)(\state, \action) &\triangleq \sum_{\mathcal{C}\in 2^{\mathcal{A}}}\mathbb{P}(\mathcal{C})\sum_{\state' \in \sset} \mathcal{P}(\state' \mid \state, \action) \nonumber \\ 
&\quad \quad \quad \left[r(\state, \action)+\gamma \max _{b \in \mathcal{C} } q(\state', b) \right].
\end{align}
We then demonstrate that it is a contraction in the sup-norm, as shown in \cref{contraction}, which we prove in Appendix \ref{proof_lemma_contraction}.
\begin{lemma}
\label{contraction}
The operator $\Phi$, defined in Eq. \eqref{Phi_operator},
is a contraction in the sup-norm, with a contraction factor $\gamma$, i.e.,
$
\left\|\Phi q_1-\Phi q_2\right\|_{\infty} \leq \gamma\left\|q_1-q_2\right\|_{\infty}.
$
\end{lemma}

We then use the above lemma to establish the convergence of Stochastic Q-learning. Given any initial estimate $Q_0$, using the considered update rule for Stochastic Q-learning, subtracting from both sides $Q^*\left(\state_t, \action_t\right)$ and letting
$\Delta_t(\state, a) \triangleq Q_t(\state, a)-Q^*(\state, a)$, yields
\begin{align}
\Delta_{t+1}\left(\state_t, \action_t\right) &= \left(1-\alpha_t\left(\state_t, \action_t\right)\right) \Delta_t\left(\state_t, \action_t\right) \nonumber \\
&\quad + \alpha_t(\state_t, \action_t)F_t(\state_t, \action_t), \nonumber \text{ with }
\end{align}
\begin{equation}
F_t(\state, \action) \triangleq r(\state, \action)+\gamma \max _{b \in \mathcal{C}} Q_t\left(\state', b\right)-Q^*\left(\state, \action\right).
\end{equation}
With $\mathcal{F}_t$ representing the past at time $t$,
$$
\begin{aligned}
\mathbb{E}\left[F_t(\state, \action) \mid \mathcal{F}_t\right] 
&=\sum_{\mathcal{C}\in 2^{\mathcal{A}}}\mathbb{P}(\mathcal{C}) \sum_{\state' \in \sset} \mathcal{P}(\state' \mid \state, \action)\left[F_t(\state, \action)\right] \\
& =\left(\Phi (Q_t) \right)(\state, \action)- Q^*(\state, \action).
\end{aligned}
$$
Using the fact that $Q^*=\Phi Q^*$ and Lemma \ref{contraction},
\begin{equation}
\label{eq:condition_1}
\left\|\mathbb{E}\left[F_t(\state, \action) \mid \mathcal{F}_t\right]\right\|_{\infty} \leq \gamma\left\|Q_t-Q^*\right\|_{\infty} =\gamma\left\|\Delta_t\right\|_{\infty}.  
\end{equation}
Given that $r$ is bounded,  its variance is bounded by some constant $B$. Thus, as shown in Appendix \ref{theorem_proof_appendix}, for $C = \max\{B  + \gamma^2  \|Q^*\|_{\infty}^2, \gamma^2\}$,
$\label{eq:condition_2}
 \operatorname{var}\left[F_t(\state, \action) \mid \mathcal{F}_t\right] \leq C (1+\|\Delta_t\|_{\infty})^2.$
Then, by this inequality, Eq. \eqref{eq:condition_1}, and Theorem 1 in \citep{jaakkola1993convergence}, $\Delta_t$ converges to zero with probability 1, i.e., $Q_t$ converges to $Q^*$ with probability 1.

\section{Experiments}

We compare stochastic maximization to exact maximization and evaluate the proposed RL algorithms in Gymnasium \citep{brockman2016openai} and MuJoCo \citep{todorov2012mujoco} environments.
The stochastic tabular Q-learning approaches are tested on CliffWalking-v0, FrozenLake-v1, and a generated MDP environment. Additionally, the stochastic deep Q-network approaches are tested on control tasks and compared against their deterministic counterparts, as well as against DDPG \cite{lillicrap2015continuous}, A2C \cite{mnih2016asynchronous}, and PPO \cite{schulman2017proximal}, using Stable-Baselines implementations \cite{stable-baselines}, which can directly handle continuous action spaces. Further details can be found in Appendix \ref{appendix_exp_details}.

\subsection{Stochastic Q-learning Average Return}
\label{Stochastic_Q_learning_Reward_Analysis}

We test Stochastic Q-learning, Stochastic Double Q-learning, and Stochastic Sarsa in environments with discrete states and actions. Interestingly, as shown in Fig. \ref{cumul_reward:frozen}, our stochastic algorithms outperform their deterministic counterparts. Furthermore, we observe that Stochastic Q-learning outperforms all the methods considered regarding the cumulative rewards in the FrozenLake-v1. Moreover, in the CliffWalking-v0 (as shown in Fig. \ref{ap:fig:clif_walking}), as well as for the generated MDP environment with 256 actions (as shown in Fig. \ref{ap:fig:mdp}), all the stochastic and non-stochastic methods reach the optimal policy in a similar number of steps. 

\begin{figure}[t]
\begin{center}
\includegraphics[width=6cm]{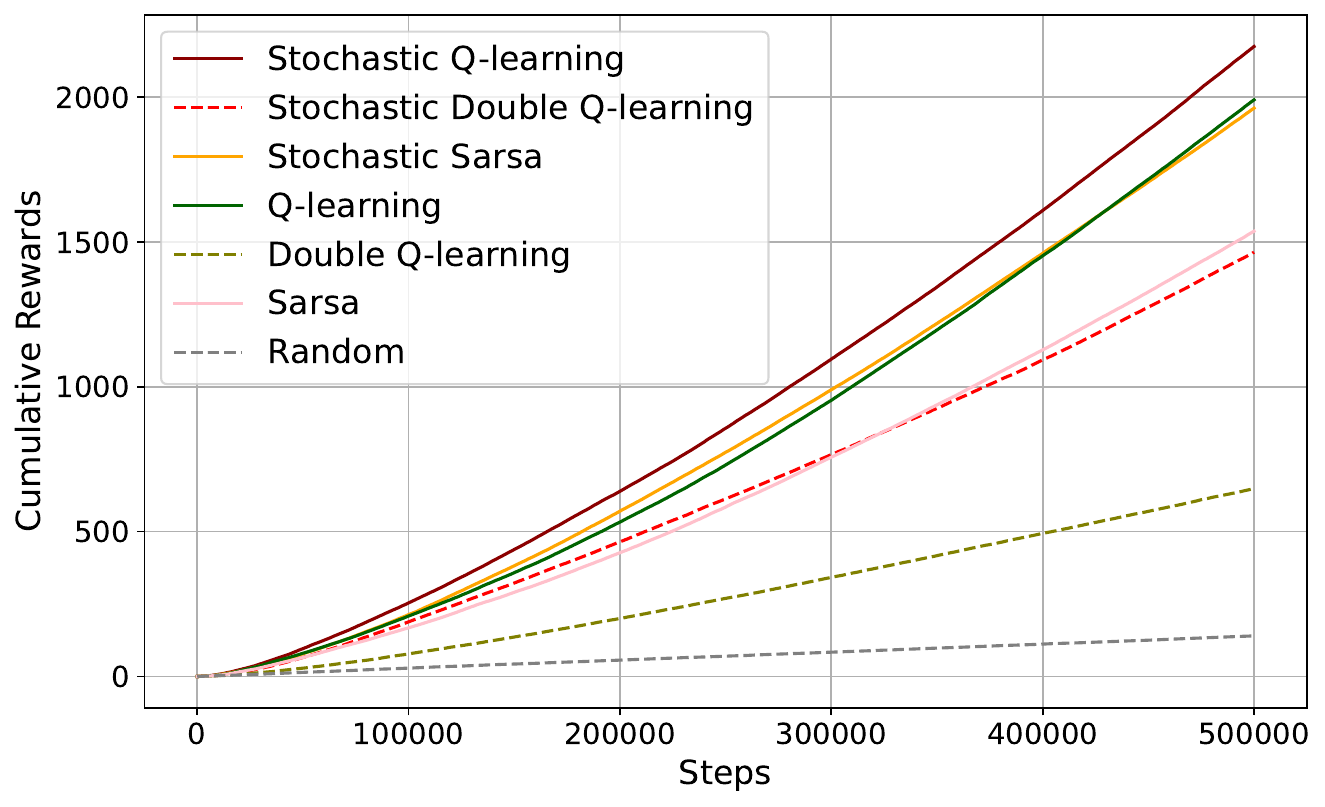}
\end{center}
\caption{\small Comparison of stochastic vs. non-stochastic value-based variants on the FrozenLake-v1, with steps on the x-axis and cumulative rewards on the y-axis.}
\label{cumul_reward:frozen}
\vspace{-.2in}
\end{figure}

\subsection{Exponential Wall Time Speedup}
\label{exponential_time_speedup}

\begin{minipage}[t]{0.21\textwidth}
Stochastic maximization methods exhibit logarithmic complexity regarding the number of actions. Therefore, StochDQN and StochDDQN, which apply these techniques for action selection and updates, have exponentially faster execution times than DQN and DDQN, as confirmed in Fig.~\ref{fig:selection_time}. 
\end{minipage}
\hspace{0.01\textwidth}
\begin{minipage}[t]{0.25\textwidth}
\vspace{-.4in}
\begin{figure}[H]
\centering
\includegraphics[width=4.7cm]{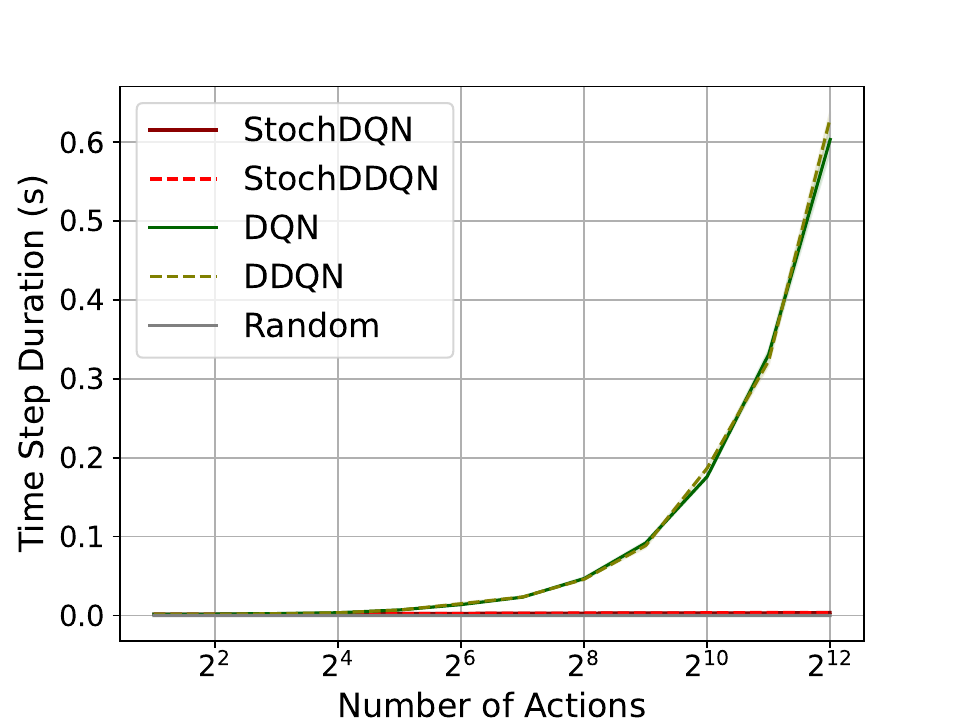}
\caption{\small Comparison of wall time in seconds of stochastic and non-stochastic DQN methods on various action set sizes.}
\label{fig:selection_time}
\end{figure}
\vspace{-.1 in}
\end{minipage}
For the time duration of action selection alone, please refer to Appendix~\ref{appendix_experiment_time}. The time analysis results show that the proposed methods are nearly as fast as a random algorithm that selects actions randomly. Specifically, in the experiments with the InvertedPendulum-v4, the stochastic methods took around 0.003 seconds per step for a set of 1000 actions, while the non-stochastic methods took 0.18 seconds, which indicates that the stochastic versions are 60 times faster than their deterministic counterparts. Furthermore, for the  HalfCheetah-v4 experiment, we considered 4096 actions, where one (D)DQN step takes 0.6 seconds, needing around 17 hours to run for 100,000 steps, while the Stoch(D)DQN needs around 2 hours to finish the same 100,000 steps. In other words, we can easily run for 10x more steps in the same period (seconds). This makes the stochastic methods more practical, especially with large action spaces.

\subsection{Stochastic Deep Q-network Average Return}
\label{Stochastic_DQN_Reward_Analysis}

Fig. \ref{reward:pendulum} shows the performance of various RL algorithms on the InvertedPendulum-v4 task, which has 512 actions. StochDQN achieves the optimal average return in fewer steps than DQN, with a lower per-step time advantage (as shown in \cref{exponential_time_speedup}). Interestingly, while DDQN struggles, StochDDQN nearly reaches the optimal average return, demonstrating the effectiveness of stochasticity. StochDQN and StochDDQN significantly outperform DDQN, A2C, and PPO by obtaining higher average returns in fewer steps. 
Similarly, Fig. \ref{reward:ap:cheetah} in \cref{appendix_experiement_dqn_rewards} shows the results for the HalfCheetah-v4 task, which has 4096 actions. Stochastic methods, particularly StochDDQN, achieve results comparable to the non-stochastic methods. Notably, all DQN methods (stochastic and non-stochastic) outperform PPO and A2C, highlighting their efficiency in such scenarios.

\begin{remark}
While comparing them falls outside the scope of our work, we note that DDQN was proposed to mitigate the inherent overestimation in DQN. However, exchanging overestimation for underestimation bias is not always beneficial, as our results demonstrate and as shown in other studies such as \cite{lan2020maxmin}.
\end{remark}

\begin{figure}[t]
\begin{center}
\includegraphics[width=6.6cm]{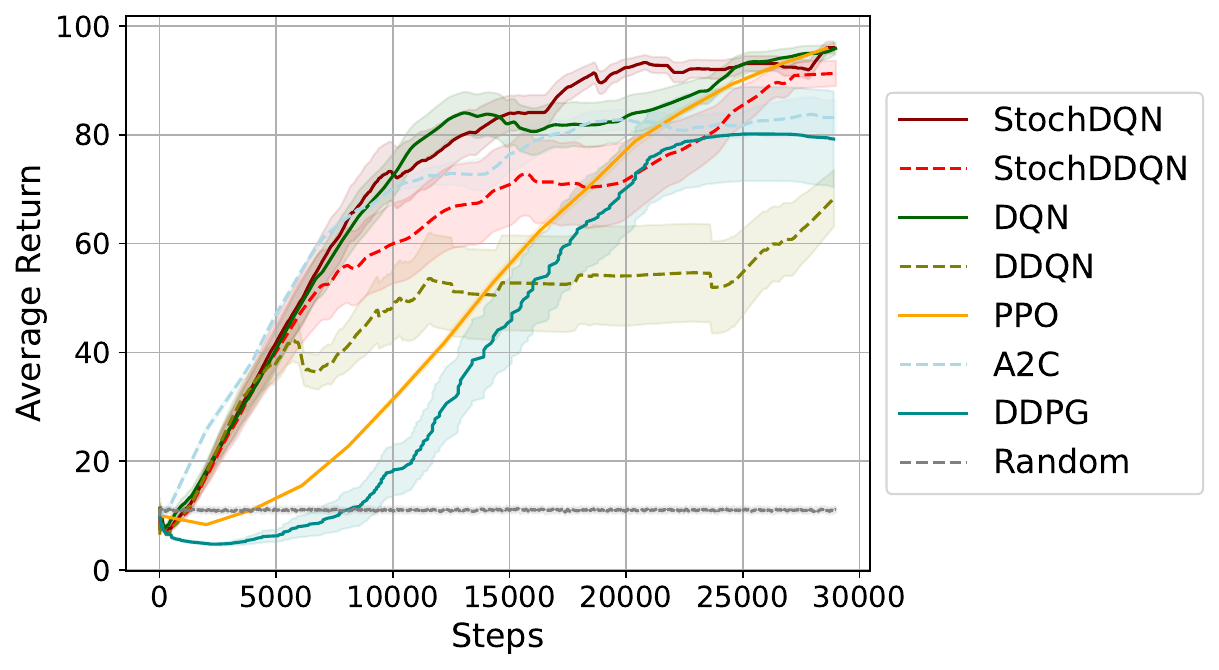}
\end{center}
\caption{\small Comparison of stochastic DQN variants against other RL algorithms on the InvertedPendulum-v4, with 1000 actions, with steps on the x-axis and average returns on the y-axis.}
\label{reward:pendulum}
\vspace{-.2in}
\end{figure}

\subsection{Stochastic Maximization}
\label{stochmax_experiments}

\begin{minipage}[t]{0.21\textwidth}
This section analyzes stochastic maximization by tracking returned values of $\omega_t$ (Eq. \eqref{omega_t_def}) across the steps. As shown in Fig. \ref{omegaandbeta}, for StochDQN, for the InvertedPendulum-v4, $\omega_t$ approaches one rapidly, similarly for the HalfCheetah-v4, as shown in Appendix \ref{appendix_experiement_stochmax}.
\end{minipage}
\hspace{0.01\textwidth}
\begin{minipage}[t]{0.25\textwidth}
\vspace{-.3in}
\begin{figure}[H]
\centering
\includegraphics[width=4.7cm]{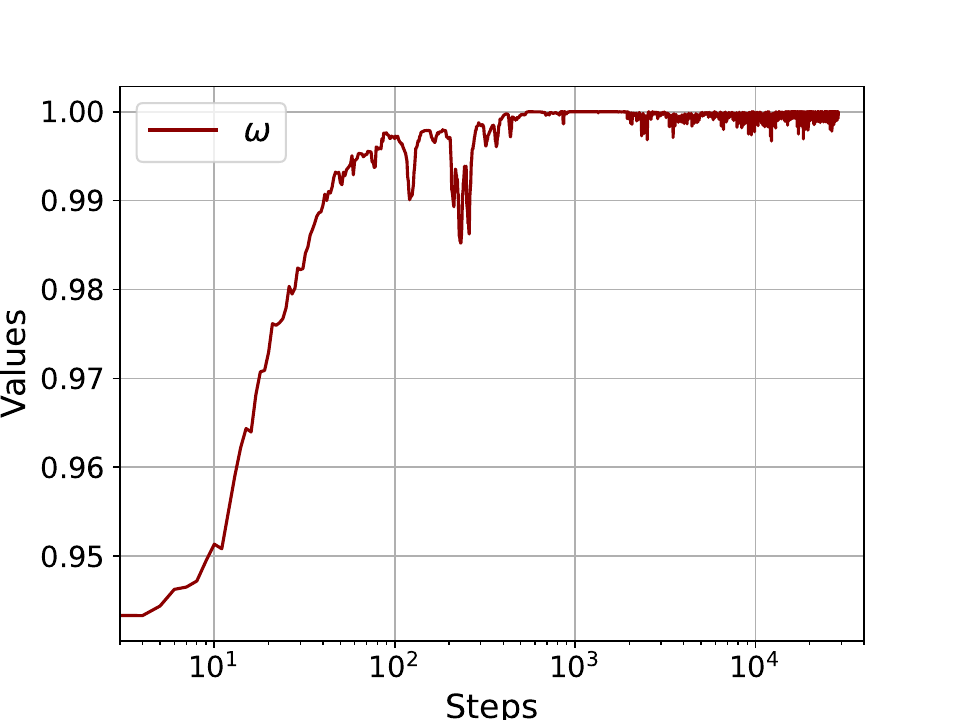}
\caption{\small The $\stochmax$ and $\max$ ratio values tracked over the steps on the InvertedPendulum-v4.}
\label{omegaandbeta}
\end{figure}
\vspace{-.1in}
\end{minipage}
Furthermore, we track the returned values of the difference $\beta_t$ (Eq. \eqref{beta_t_def}) and show that it is bounded by small values in both environments, as illustrated in Appendix \ref{appendix_experiement_stochmax}.

\section{Discussion}

In this work, we focus on adapting value-based methods, which excel in generalization compared to actor-based approaches \cite{dulac2015deep}. However, this advantage comes at the cost of lower computational efficiency due to the maximization operation required for action selection and value function updates. Therefore, our primary motivation is to provide a computationally efficient alternative for situations with general large discrete action spaces.

We focus mainly on Q-learning-like methods among value-based approaches due to their off-policy nature and proven success in various applications. We demonstrate that these methods can be applied to large discrete action spaces while achieving exponentially lower complexity and maintaining good performance. Furthermore, our proposed stochastic maximization method performs well even when applied to the on-policy Sarsa algorithm, extending its potential beyond off-policy methods. Consequently, the suggested stochastic approach offers broader applicability to other value-based approaches, resulting in lower complexity and improved efficiency with large discrete action spaces.

While the primary goal of this work is to reduce the complexity and wall time of Q-learning-like algorithms, our experiments revealed that stochastic methods not only achieve shorter step times (in seconds) but also, in some cases, yield higher rewards and exhibit faster convergence in terms of the number of steps compared to other methods. 
These improvements can be attributed to several factors. Firstly, introducing more stochasticity into the greedy choice through $\stochargmax$ enhances exploration. Secondly, Stochastic Q-learning specifically helps to reduce the inherent overestimation in Q-learning-like methods \citep{hasselt2010double, lan2020maxmin, wang2021adaptive}. This reduction is achieved using $\stochmax$, a lower bound to the $\max$ operation.

Q-learning methods, focused initially on discrete actions, can be adapted to tackle continuous problems with discretization techniques and stochastic maximization. Our control experiments show that Q-network methods with discretization achieve superior performance to algorithms with continuous actions, such as PPO, by obtaining higher rewards in fewer steps, which aligns with observations in previous works that highlight the potential of discretization for solving continuous control problems \cite{dulac2015deep, tavakoli2018action, tang2020discretizing}. Notably, the logarithmic complexity of the proposed stochastic methods concerning the number of considered actions makes them well-suited for scenarios with finer-grained discretization, leading to more practical implementations.

\section{Conclusion}
We propose adapting Q-learning-like methods to mitigate the computational bottleneck associated with the $\max$ and $\argmax$ operations in these methods. By reducing the maximization complexity from linear to sublinear using $\stochmax$ and $\stochargmax$, we pave the way for practical and efficient value-based RL for large discrete action spaces. We prove the convergence of Stochastic Q-learning, analyze stochastic maximization, and empirically show that it performs well with significantly low complexity.


\section*{Impact Statement}
This paper presents work whose goal is to advance the field of Machine Learning. There are many potential societal consequences of our work, none which we feel must be specifically highlighted here.

\bibliographystyle{icml2024}

\bibliography{main}

\newpage
\onecolumn
\appendix

\section{Stochastic Q-learning Convergence Proofs}
\label{appendix_theorem_proof}

In this section, we prove Theorem \ref{theorem}, which states the convergence of Stochastic Q-learning. This algorithm uses a stochastic policy for action selection, employing a $\stochargmax$ with or without memory, possibly dependent on the current state $\state$. For value updates, it utilizes a $\stochmax$ without memory, independent of the following state $\state'$.

\subsection{Proof of Theorem \ref{theorem}}
\label{theorem_proof_appendix}

\begin{proof}
Stochastic Q-learning employs a stochastic policy in a given state $\state$, which use $\stochargmax$ operation, with or without memory $\mathcal{M}$, with probability $(1-\varepsilon_{\state})$, for $\varepsilon_{\state} > 0$, which can be summarized by the following equation:
\begin{equation}
\label{ap:eq:stochargmaxPolicy}
\pi^S_Q(\state) =
\begin{cases}
\text{play randomly} & \text{with probability } \epsilon_{\state} \\
\stochargmax_{\action \in \aset} Q(\state,\action) & \text{otherwise }. 
\end{cases}
\end{equation}
This policy, with $\varepsilon_{\state}>0$, ensures that $\mathbb{P}_{\pi}[\action_t = \action \mid \state_t = \state] > 0$ for all $(\state,\action) \in \sset \times \aset$. 

Furthermore, during the training, to update the Q-function, given any initial estimate $Q_0$, we consider a Stochastic Q-learning which uses $\stochmax$ operation as in the following stochastic update rule:
\begin{align}
&Q_{t+1}\left(\state_t, \action_t\right)=\left(1-\alpha_t\left(\state_t, \action_t\right)\right) Q_t\left(\state_t, \action_t\right) +\alpha_t\left(\state_t, \action_t\right)\left[r_t+ \gamma\stochmax _{b \in \mathcal{A}} Q_t\left(\state_{t+1}, b\right)\right] .    
\end{align}
For the function updates, we consider a $\stochmax$ without memory, which involves a $\max$ over a random subset of action $\mathcal{C}$ sampled from a set probability distribution $\mathbb{P}$ defined over the combinatorial space of actions, i.e., $\mathbb{P}: 2^\mathcal{A} \rightarrow [0,1]$, which can be a uniform distribution over the action sets of size $\lceil\log(n)\rceil$. 

Hence, for a random subset of actions $\mathcal{C}$, the update rule of Stochastic Q-learning can be written as:
\begin{align}
&Q_{t+1}\left(\state_t, \action_t\right)=\left(1-\alpha_t\left(\state_t, \action_t\right)\right) Q_t\left(\state_t, \action_t\right) +\alpha_t\left(\state_t, \action_t\right)\left[r_t+ \gamma \max _{b \in \mathcal{C}} Q_t\left(\state_{t+1}, b\right)\right] .    
\end{align}
We define an optimal Q-function, denoted as $Q^*$, as follows:
\begin{align}
Q^*(\state, \action) &= \mathbb{E}\left[r(\state, \action)+\gamma \text{ stoch}\max _{b \in \mathcal{A}} Q^*(\state', b) \mid \state, \action \right]\\
&= \mathbb{E}\left[r(\state, \action)+\gamma \max _{b \in \mathcal{C}} Q^*(\state', b) \mid \state, \action \right].
\end{align}
Subtracting from both sides $Q^*\left(\state_t, \action_t\right)$ and letting
\begin{align}
\Delta_t(\state, \action)=Q_t(\state, \action)-Q^*(\state, \action),
\end{align}
yields
\begin{align}
\Delta_{t+1}\left(\state_t, \action_t\right)= \left(1-\alpha_t\left(\state_t, \action_t\right)\right) \Delta_t\left(\state_t, \action_t\right) + \alpha_t(\state_t, \action_t)F_t(\state_t, \action_t),
\end{align}
with
\begin{equation}
F_t(\state, \action)=r(\state, \action)+\gamma \max _{b \in \mathcal{C}} Q_t\left(\state', b\right)-Q^*\left(\state, \action\right).
\end{equation}

For the transition probability distribution $\mathcal{P}: \sset \times \aset \times \sset \rightarrow [0,1]$, the set probability distribution $\mathbb{P}: 2^\mathcal{A} \rightarrow [0,1]$, the reward function $r: \sset \times \aset \rightarrow \reals$, and the discount factor, $\gamma \in [0,1]$, we define the following contraction operator $\Phi$, defined for a function $q: \sset \times \aset \rightarrow \reals$ as
\begin{equation}
 (\Phi q)(\state, \action)=\sum_{\mathcal{C}\in 2^{\mathcal{A}}}\mathbb{P}(\mathcal{C})\sum_{\state' \in \sset} \mathcal{P}(\state' \mid \state, \action)\left[r(\state, \action)+\gamma \max _{b \in \mathcal{C}} q(\state', b) \right].
\end{equation}
Therefore, with $\mathcal{F}_t$ representing the past at time step $t$,
$$
\begin{aligned}
\mathbb{E}\left[F_t(\state, \action) \mid \mathcal{F}_t\right] &=\sum_{\mathcal{C}\in 2^{\mathcal{A}}}\mathbb{P}(\mathcal{C}) \sum_{\state' \in \sset} \mathcal{P}(\state' \mid \state, \action)\left[r(\state, \action)+\gamma \max _{b \in \mathcal{C}} Q_t\left(\state', b\right)-Q^*\left(\state, \action\right)\right] \\
& =\left(\Phi (Q_t) \right)(\state, \action)- Q^*(\state, \action).
\end{aligned}
$$
Using the fact that $Q^*=\Phi Q^*$,
$$
\mathbb{E}\left[F_t(\state, \action) \mid \mathcal{F}_t\right]=\left(\Phi Q_t\right)(\state, \action)-\left(\Phi Q^*\right)(\state, \action) .
$$

It is now immediate from Lemma \ref{contraction}, which we prove in Appendix \ref{proof_lemma_contraction}, that
\begin{equation}
\label{eq:app:condition_1}
\left\|\mathbb{E}\left[F_t(\state, \action) \mid \mathcal{F}_t\right]\right\|_{\infty} \leq \gamma\left\|Q_t-Q^*\right\|_{\infty} =\gamma\left\|\Delta_t\right\|_{\infty}. 
\end{equation}
Moreover,
$$
\begin{aligned}
 \operatorname{var}\left[F_t(\state, \action) \mid \mathcal{F}_t\right] 
& =\mathbb{E}\left[\left(r(\state, \action)+\gamma \max _{b \in \mathcal{C}} Q_t(\state', b)-Q^*(\state, \action)-\left(\Phi Q_t\right)(\state, \action)+Q^*(\state, \action)\right)^2\mid \mathcal{F}_t\right] \\
& =\mathbb{E}\left[\left(r(\state, \action)+\gamma \max _{b \in \mathcal{C}} Q_t(\state', b)-\left(\Phi Q_t\right)(\state, \action)\right)^2\mid \mathcal{F}_t\right] \\
& =\operatorname{var}\left[r(\state, \action)+\gamma \max _{b \in \mathcal{C}} Q_t(\state', b) \mid \mathcal{F}_t\right] \\
&=\operatorname{var}\left[r(\state, \action)\mid \mathcal{F}_t\right]+\gamma^2\operatorname{var}\left[ \max _{b \in \mathcal{C}} Q_t(\state', b) \mid \mathcal{F}_t\right]+2\gamma\operatorname{cov}(r(\state, \action),\max _{b \in \mathcal{C}} Q_t(\state', b)\mid \mathcal{F}_t)\\
&=\operatorname{var}\left[r(\state, \action)\mid \mathcal{F}_t\right]+\gamma^2\operatorname{var}\left[ \max _{b \in \mathcal{C}} Q_t(\state', b) \mid \mathcal{F}_t\right].
\end{aligned}
$$
The last line follows from the fact that the randomness of $\max _{b \in \mathcal{C}} Q_t(\state', b)\mid \mathcal{F}_t$ only depends on the random set $\mathcal{C}$ and the next state $\state'$. Moreover, we consider the reward $r(\state, \action)$ independent of the set $\mathcal{C}$ and the next state $\state'$, by not using the same set $\mathcal{C}$ for both the action selection and the value update. 

Given that $r$ is bounded, its variance is bounded by some constant $B$. Therefore, 
$$
\begin{aligned}
 \operatorname{var}\left[F_t(\state, \action) \mid \mathcal{F}_t\right] 
& \leq B +\gamma^2\operatorname{var}\left[ \max _{b \in \mathcal{C}} Q_t(\state', b) \mid \mathcal{F}_t\right]\\
& = B + \gamma^2 \mathbb{E}\left[ (\max _{b \in \mathcal{C}} Q_t(\state', b))^2 \mid \mathcal{F}_t\right] - \gamma^2 \mathbb{E}\left[ \max _{b \in \mathcal{C}} Q_t(\state', b) \mid \mathcal{F}_t\right]^2 \\
& \leq B + \gamma^2 \mathbb{E}\left[ (\max _{b \in \mathcal{C}} Q_t(\state', b))^2 \mid \mathcal{F}_t\right] \\
& \leq B + \gamma^2 \mathbb{E}\left[ (\max_{\state' \in \sset}\max _{b \in \mathcal{A}} Q_t(\state', b))^2 \mid \mathcal{F}_t\right] \\
& \leq B + \gamma^2 (\max_{\state' \in \sset}\max _{b \in \mathcal{A}} Q_t(\state', b))^2  \\
& = B + \gamma^2 \|Q_t\|_{\infty}^2  \\
& = B + \gamma^2 \|\Delta_t + Q^*\|_{\infty}^2  \\
& \leq B  + \gamma^2  \|Q^*\|_{\infty}^2 + \gamma^2 \|\Delta_t\|_{\infty}^2 \\
& \leq (B  + \gamma^2  \|Q^*\|_{\infty}^2)(1+\|\Delta_t\|_{\infty}^2) + \gamma^2 (1+\|\Delta_t\|_{\infty}^2) \\
& \leq \max\{B  + \gamma^2  \|Q^*\|_{\infty}^2, \gamma^2\}(1+\|\Delta_t\|_{\infty}^2) \\
& \leq \max\{B  + \gamma^2  \|Q^*\|_{\infty}^2, \gamma^2\}(1+\|\Delta_t\|_{\infty})^2. \\
\end{aligned}
$$
Therefore, for constant $C = \max\{B  + \gamma^2  \|Q^*\|_{\infty}^2, \gamma^2\}$,
\begin{equation}
\label{eq:app:condition_2}
 \operatorname{var}\left[F_t(\state, \action) \mid \mathcal{F}_t\right] \leq C (1+\|\Delta_t\|_{\infty})^2.
\end{equation}

Then, by Eq. \eqref{eq:app:condition_1}, Eq. \eqref{eq:app:condition_2}, and Theorem 1 in \citep{jaakkola1993convergence}, $\Delta_t$ converges to zero with probability 1, i.e., $Q_t$ converges to $Q^*$ with probability 1.
\end{proof}

\subsection{Proof of Lemma \ref{contraction}}
\label{proof_lemma_contraction}

\begin{proof}
For the transition probability distribution $\mathcal{P}: \sset \times \aset \times \sset \rightarrow [0,1]$, the set probability distribution $\mathbb{P}$ defined over the combinatorial space of actions, i.e., $\mathbb{P}: 2^\mathcal{A} \rightarrow [0,1]$, the reward function $r: \sset \times \aset \rightarrow \reals$, and the discount factor $\gamma \in [0,1]$, for a function $q: \sset \times \aset \rightarrow \reals$, the operator $\Phi$ is defined as follows:
\begin{align}
(\Phi q)(\state, \action) =\sum_{\mathcal{C}\in 2^{\mathcal{A}}}\mathbb{P}(\mathcal{C})\sum_{\state' \in \sset} \mathcal{P}(\state' \mid \state, \action)\left[r(\state, \action)+\gamma \max _{b \in \mathcal{C}} q(\state', b) \right] .
\end{align}
Therefore, 
$$
\begin{aligned}
\left\|\Phi q_1-\Phi q_2\right\|_{\infty} &=\max _{\state, \action}\left|\sum_{\mathcal{C}\in 2^{\mathcal{A}}}\mathbb{P}(\mathcal{C})\sum_{\state' \in \sset} \mathcal{P}(\state' \mid \state, \action)\left[r(\state, \action) +\gamma \max _{b \in \mathcal{C}} q_1(\state', b)-r(\state, \action)+\gamma \max _{b \in \mathcal{C}} q_2(\state', b)\right]\right| \\
& =\max _{\state, \action} \gamma\left|\sum_{\mathcal{C}\in 2^{\mathcal{A}}}\mathbb{P}(\mathcal{C})\sum_{\state' \in \sset} \mathcal{P}(\state' \mid \state, \action)\left[\max _{b \in \mathcal{C}} q_1(\state', b)-\max _{b \in \mathcal{C}} q_2(\state', b)\right]\right|  \\
& \leq \max _{\state, \action} \gamma \sum_{\mathcal{C}\in 2^{\mathcal{A}}}\mathbb{P}(\mathcal{C})\sum_{\state' \in \sset} \mathcal{P}(\state' \mid \state, \action)\left|\max _{b \in \mathcal{C}} q_1(\state', b)-\max _{b \in \mathcal{C}} q_2(\state', b)\right|  \\
& \leq \max _{\state, \action} \gamma \sum_{\mathcal{C}\in 2^{\mathcal{A}}}\mathbb{P}(\mathcal{C})\sum_{\state' \in \sset} \mathcal{P}(\state' \mid \state, \action) \max _{z, b}\left|q_1(z, b)-q_2(z, b)\right| \\
& \leq\max _{\state, \action} \gamma \sum_{\mathcal{C}\in 2^{\mathcal{A}}}\mathbb{P}(\mathcal{C})\sum_{\state' \in \sset} \mathcal{P}(\state' \mid \state, \action)\left\|q_1-q_2\right\|_{\infty} \\
& =\gamma\left\|q_1-q_2\right\|_{\infty} .
\end{aligned}
$$ 
\end{proof}

\newpage
\section{Stochastic Maximization}
\label{sec_stoch_max_app}


We analyze the proposed stochastic maximization method by comparing its error to that of exact maximization. First, we consider the case without memory, where $\mathcal{C}=\mathcal{R}$, and then the case with memory, where $\mathcal{M} \neq \emptyset$. Finally, we provide a specialized bound for the case where the action values follow a uniform distribution.

\subsection{Memoryless Stochastic Maximization}

In the following lemma, we give a lower bound on the probability of finding an optimal action within a uniformly sampled subset $\mathcal{R}$ of $\lceil\log(n)\rceil$ actions. We prove that for a given state $s$, the probability $p$ of sampling an optimal action within the uniformly randomly sampled subset $\mathcal{R}$ of size $\lceil\log(n)\rceil$ actions is lower bounded with $p \geq \frac{\lceil\log(n)\rceil}{n}$. 

\subsubsection{Proof of Lemma \ref{proba_lemma}}
\label{prooflemma}
\begin{proof}
In the presence of multiple maximizers, we focus on one of them, denoted as $\action^*_0$, and then the probability $p$ of sampling at least one maximizer is lower-bounded by the probability $p_{\action^*_0}$ of finding $\action^*_0$, i.e., 
$$p \geq p_{\action^*_0}.$$

The probability $p_{\action^*_0}$ of finding $\action^*_0$ is the probability of sampling $\action^*_0$ within the random set $\mathcal{R}$ of size $\lceil\log(n)\rceil$, which is the fraction of all possible combinations of size $\lceil\log(n)\rceil$ that include $\action^*_0$. 

This fraction can be calculated as ${n-1 \choose \lceil\log(n)\rceil - 1}$ divided by all possible combinations of size $\lceil\log(n)\rceil$, which is ${n \choose \lceil\log(n)\rceil}$. 

Therefore, $p_{\action^*_0} = \frac{{n-1 \choose \lceil\log(n)\rceil - 1}}{{n \choose \lceil\log(n)\rceil}}$.

Consequently,
\begin{equation}
p \geq \frac{\lceil\log(n)\rceil}{n}.
\end{equation}
\end{proof}



\subsection{Stochastic Maximization with Memory}

While stochastic maximization without memory could approach the maximum value or find it with the probability $p$, lower-bounded in Lemma \ref{proba_lemma}, it never converges to an exact maximization, as it keeps sampling purely at random, as can be seen in Fig. \ref{reward:uniform}. However, stochastic maximization with memory can become an exact maximization when the Q-function becomes stable, which we prove in the following Corollary. Although the $\stochmax$ has sub-linear complexity compared to the max, the following Corollary shows that, on average, for a stable Q-function, after a certain number of iterations, the output of the $\stochmax$ matches the output of max.

\begin{definition}
A Q-function is considered stable for a given state $\state$ if its best action in that state remains unchanged for all subsequent steps, even if the Q-function's values themselves change.
\end{definition}
A straightforward example of a stable Q-function occurs during validation periods when no function updates are performed. However, in general, a stable Q-function does not have to be static and might still vary over the rounds; the key characteristic is that its maximizing action remains the same even when its values are updated. Although the $\stochmax$ has sub-linear complexity compared to the $\max$, without any assumption of the value distributions, the following Corollary shows that, on average, for a stable Q-function, after a certain number of iterations, the output of the $\stochmax$ matches exactly the output of $\max$.

\subsubsection{Proof of Corollary \ref{cor_stable}}
\label{corollary_proof}

\begin{proof}
We formalize the problem as a geometric distribution where the success event is the event of sampling a subset of size $\lceil \log(n)\rceil$ that includes at least one maximizer. The geometric distribution gives the probability that the first time to sample a subset that includes an optimal action requires $k$ independent calls, each with success probability $p$. From Lemma \ref{proba_lemma}, we have $p \geq \frac{\lceil\log(n)\rceil}{n}$. Therefore, on an average, success requires: $\frac{1}{p} \leq \frac{n}{\lceil\log(n)\rceil}$ calls. 

For a given discrete state $\state$, $\mathcal{M}$ keeps track of the most recent best action found. For $\mathcal{C} = \mathcal{R} \cup \mathcal{M}$,
\begin{align}
\stochmax_{\action \in \aset} Q(\state,\action)  
&= \max_{\action \in \mathcal{C}} Q(\state,\action) \geq \max_{\action \in \mathcal{M}} Q(\state,\action).  
\end{align}
Therefore, for a given state $\state$, on average, if the Q-function is stable, then within $\frac{n}{\lceil\log(n)\rceil}$, $\mathcal{M}$ will contain the optimal action $\action^*$. 
Therefore, on an average, after $\frac{n}{\lceil\log(n)\rceil}$ time steps, 
\begin{align}
\stochmax_{\action \in \aset} Q(\state,\action) &\geq \max_{\action \in \mathcal{M}} Q(\state,\action) = \max_{\action \in \aset} Q(\state,\action). \nonumber
\end{align}
We know that, 
$ \stochmax_{\action \in \aset} Q(\state,\action) \leq \max_{\action \in \aset} Q(\state,\action). $
Therefore, for a stable Q-function, on an average, after $\frac{n}{\lceil\log(n)\rceil}$ time steps, $\stochmax_{\action \in \aset} Q(\state,\action)$ becomes $\max_{\action \in \aset} Q(\state,\action)$.
\end{proof}

\subsection{Stochastic Maximization with Uniformly Distributed Rewards}
\label{stoch_max_with_uniform}

While the above corollary outlines an upper-bound on the average number of calls needed to determine the exact optimal action eventually, the following lemma offers insights into the expected maximum value of a randomly sampled subset of actions, comprising $\lceil\log(n)\rceil$ elements when their values are uniformly distributed.

\begin{lemma}
\label{lemma_uniform}
For a given state $\state$ and a uniformly randomly sampled subset $\mathcal{R}$ of size $\lceil\log(n)\rceil$ actions, if the values of the sampled actions follow independently a uniform distribution in the interval $[Q_t(\state,\action_t^{\star})-b_t(\state), Q_t(\state,\action_t^{\star})]$, then the expected value of the maximum Q-function within this random subset is:
\begin{align}
& \mathbb{E}\left[\max_{k\in \mathcal{R}} Q_t(\state,k) \mid \state, \action^{\star}_t \right] = Q_t(\state,\action^{\star}_t) - \frac{b_t(\state)}{\lceil\log(n)\rceil+1}.
\end{align}
\end{lemma}
\begin{proof}
    For a given state $\state$ we assume a uniformly randomly sampled subset $\mathcal{R}$ of size $\lceil\log(n)\rceil$ actions, and the values of the sampled actions are independent and follow a uniform distribution in the interval $[Q_t(\state,\action_t^{\star})-b_t(\state), Q_t(\state,\action_t^{\star})]$. Therefore, the cumulative distribution function (CDF) for the value of an action $\action$ given the state $\state$ and the optimal action $\action_t^*$ is:
\[
G(y;\state,\action)
=
\left\{\begin{array}{ll}
0 &  \text{for $y < Q_t(\state,\action_t^*)-b_t(\state)$} \\
y &  \text{for $y \in [Q_t(\state,\action_t^*)-b_t,Q_t(\state,\action^*)]$} \\
1 & \text{for $y > Q_t(\state,\action_t^*)$} \,.
\end{array}\right.
\]

We define the variable $x = (y - (Q_t(\state,\action_t^*)-b_t(\state)))/b_t(\state)$.
\[
F(x;\state,\action)
=
\left\{\begin{array}{ll}
0 &  \text{for $x < 0$} \\
x &  \text{for $x \in [0,1]$} \\
1 & \text{for $x > 1$} \,.
\end{array}\right.
\]

If we select $\lceil\log(n)\rceil$ such actions, the CDF of the maximum of these actions, denoted as $F_{\max}$ is the following:
\def\Fmax{F_{\max}(x;\state,\action)}
\begin{align*}
\Fmax
& = \mathbb{P}\left(\max_{a \in \mathcal{R}} Q_t(\state,\action) \le x\right) \\
& = \prod_{a \in \mathcal{R}} \mathbb{P}\left(Q_t(\state,\action) \le x\right) \\
& = \prod_{a \in \mathcal{R}} F(x;\state,\action) \\
& = F(x;\state,\action)^{\lceil\log(n)\rceil} .\,
\end{align*}
The second line follows from the independence of the values, and the last line follows from the assumption that all actions follow the same uniform distribution.

The CDF of the maximum is therefore given by:
\[
\Fmax
=
\left\{\begin{array}{ll}
0 &  \text{for $x < 0$}  \\
x^{\lceil\log(n)\rceil} &  \text{for $x \in [0,1]$} \\
1 & \text{for $x >1$} \,.
\end{array}\right.
\]

Now, we can determine the desired expected value as
\begin{align*}
\mathbb{E}\left[\max_{\action \in \mathcal{R}} \frac{Q_t(\state,\action) - (Q_t(\state,\action_t^*)-b_t(\state))}{b_t(\state)}\right] & = \int_{-\infty}^\infty x  \, \text{d} \Fmax \\
& = \int_{0}^1 x \, \text{d} \Fmax \\
& = \left[ x \Fmax \right]_0^1 -\int_{0}^1 \Fmax \, \text{d} x \\
& =  1 - \int_{0}^1 x^{\lceil\log(n)\rceil} \, \text{d} x \\
& =  1 -  \frac{1}{\lceil\log(n)\rceil+1}.  \\
\end{align*}
We employed the identity $\int_0^1 x \,\text{d} \mu(x) = \int_0^1 1 - \mu(x) \,\text{d}x$, which can be demonstrated through integration by parts. To return to the original scale, we can first multiply by $b_t$ and then add $Q_t(\state,\action_t^*)-b_t(\state)$, resulting in:
\begin{align*}
\mathbb{E}\left[\max_{\action \in \mathcal{R}} Q_t(\state,\action) \mid \state, \action_t^* \right] & = Q_t(\state,\action_t^*) -  \frac{b_t(\state)}{\lceil\log(n)\rceil+1}. \\
\qedhere
\end{align*}
\end{proof}

As an example of this setting, for $Q_t(\state,\action^{\star}_t) =100$, $b_t=100$, for a setting with $n = 1000$ actions, $\lceil\log(n)\rceil+1=11$. Hence the $\mathbb{E}\left[\max_{k\in \mathcal{R}} Q_t(\state,k) \mid \state, \action^{\star}_t \right] \approx 91$. This shows that even with a randomly sampled set of actions $\mathcal{R}$, the $\stochmax$ can be close to the max. We simulate this setting in the experiments in Fig. \ref{reward:uniform}.

\textcolor{black}{
Our proposed stochastic maximization does not solely rely on the randomly sampled subset of actions $\mathcal{R}$ but also considers actions from previous experiences through $\mathcal{M}$. Therefore, the expected $\stochmax$ should be higher than the above result, providing an upper bound on the expected $\beta_t$ as described in the following corollary of Lemma \ref{lemma_uniform}.
\begin{corollary}
\label{corollary_beta}
For a given discrete state $\state$, if the values of the sampled actions follow independently a uniform distribution from the interval $[Q_t(\state,\action_t^{\star})-b_t(\state), Q_t(\state,\action_t^{\star})]$, then the expected value of $\beta_t(\state)$ is:
\begin{align}
& \mathbb{E}\left[\beta_t(\state) \mid \state \right] \leq \frac{b_t(\state)}{\lceil\log(n)\rceil+1}.
\end{align} 
\end{corollary}
\begin{proof}
At time step $t$, given a state $\state$, and the current estimated Q-function $Q_t$, $\beta_t(\state)$ is defined as follows:
\begin{equation}
\beta_t(\state) = \max _{\action \in \mathcal{A}} Q_t\left(\state, \action\right) - \stochmax _{\action \in \mathcal{A}} Q_t\left(\state, \action\right).
\end{equation}
For a given state $\state$ and a uniformly randomly sampled subset $\mathcal{R}$ of size $\lceil\log(n)\rceil$ actions and a subset of some previous played actions $\mathcal{M} \subset \mathcal{E}$, using the law of total expectation,
\begin{align*}
\mathbb{E}\left[\beta_t(\state)\mid\state\right]  \nonumber &= \mathbb{E}\left[\mathbb{E}\left[\beta_t(\state)\mid \state, \action^{\star}_t \right]\mid\state\right]  \nonumber \\
&= \mathbb{E}\left[\mathbb{E}\left[\max_{k\in \aset} Q_t(\state,k) - \stochmax_{k\in \aset} Q_t(\state,k)\mid \state, \action^{\star}_t \right]\mid\state\right]  \nonumber \\
&=  \mathbb{E}\left[\mathbb{E}\left[\max_{k\in \aset} Q_t(\state,k) - \max_{k\in \mathcal{R} \cup \mathcal{M}} Q_t(\state,k)\mid \state, \action^{\star}_t \right]\mid\state\right] \nonumber \\
&\leq \mathbb{E}\left[\mathbb{E}\left[ \max_{k\in \aset} Q_t(\state,k) - \max_{k\in \mathcal{R}} Q_t(\state,k)\mid \state, \action^{\star}_t \right]\mid\state\right] \nonumber \\
&= \mathbb{E}\left[Q_t(\state,\action_t^*) - \mathbb{E}\left[ \max_{k\in \mathcal{R}} Q_t(\state,k)\mid \state, \action^{\star}_t \right]\mid\state\right]. \nonumber
\end{align*}
Therefore by Lemma \ref{lemma_uniform}:
\begin{align*}
\mathbb{E}\left[\beta_t(\state)\mid\state\right]  \nonumber &\leq \mathbb{E}\left[Q_t(\state,\action_t^*) - (Q_t(\state,\action_t^*) -  \frac{b_t(\state)}{\lceil\log(n)\rceil+1})\mid\state\right] \nonumber \\
&= \mathbb{E}\left[\frac{b_t(\state)}{\lceil\log(n)\rceil+1}\mid\state\right] \nonumber \\
&= \frac{b_t(\state)}{\lceil\log(n)\rceil+1}. \nonumber
\end{align*} 
\end{proof}}

\newpage
\section{Pseudocodes}
\label{appendix_pseudo_codes}

\begin{algorithm}
\begin{algorithmic}
\caption{Stochastic Double Q-learning}
\label{alg:stoc_double_q_learning}
\STATE Initialize $Q^A(\state, \action)$ and $Q^B(\state, \action)$ for all $\state \in \mathcal{S}, \action \in \mathcal{A}$, $n=|\aset|$
\FOR{each episode}
    \STATE Observe state $\state$.
    \FOR{each step of episode}
        \STATE Choose $\action$ from $\state$ via $Q^A+Q^B$ with policy $\pi^S_{(Q^A+Q^B)}(\state)$ in Eq. \eqref{eq:stochargmaxPolicy}.
        \STATE Take action $\action$, observe $r$, $\state'$.
        \STATE Choose either UPDATE(A) or UPDATE(B), for example randomly.
        \IF{UPDATE(A)}
            \STATE $\Delta^A \leftarrow r + \gamma  Q^B(\state', \stochargmax_{b\in \mathcal{A}} Q^A(\state',b)) - Q^A(\state, \action)$.
            \STATE $Q^A(\state, \action) \leftarrow Q^A(\state, \action) + \alpha(\state,\action) \Delta^A$.\\
        \ELSIF{UPDATE(B)}
            \STATE $\Delta^B \leftarrow r + \gamma  Q^A(\state', \stochargmax_{b\in \mathcal{A}} Q^B(\state',b)) - Q^B(\state, \action)$.
            \STATE $Q^B(\state, \action) \leftarrow Q^B(\state, \action) + \alpha(\state,\action) \Delta^B$.\\
        \ENDIF
        \STATE $\state \leftarrow \state'$.
    \ENDFOR
\ENDFOR
\end{algorithmic}
\end{algorithm}

\begin{algorithm}
\begin{algorithmic}
\caption{Stochastic Sarsa}
\label{alg:sarsa}
\STATE Initialize $Q(\state, \action)$ for all $\state \in \mathcal{S}, \action \in \mathcal{A}$, $n=|\aset|$
\FOR{each episode}
    \STATE Observe state $\state$.
    \STATE Choose $\action$ from $\state$ with policy $\pi_Q^S(\state)$ in Eq. \eqref{eq:stochargmaxPolicy}.
    \FOR{each step of episode}
        \STATE Take action $\action$, observe $r$, $\state'$.
        \STATE Choose $\action'$ from $\state'$ with policy $\pi_Q^S(\state')$ in Eq. \eqref{eq:stochargmaxPolicy}.
        \STATE $Q(\state, \action) \leftarrow Q(\state, \action) + \alpha(\state,\action) [r + \gamma  Q(\state', \action') - Q(\state, \action)]$.\\
        \STATE $\state \leftarrow \state'$; $\action \leftarrow \action'$.
    \ENDFOR
\ENDFOR
\end{algorithmic}
\end{algorithm}

\begin{algorithm}
\begin{algorithmic}
\caption{Stochastic Deep Q-Network (StochDQN)}
\label{alg:StochDQN}
\STATE \textbf{Algorithm parameters:} learning rate $\alpha \in (0, 1]$, replay buffer $\mathcal{E}$, update rate $\tau$. \\
\STATE \textbf{Initialize:} neural network $Q(\state, \action; \theta)$ with random weights $\theta$, target network $\hat{Q}(\state, \action; \theta^-)$ with $\theta^- = \theta$, set of actions $\aset$ of size $n$.\\
\FOR{each episode}
    \STATE Initialize state $\state$.
    \WHILE{not terminal state is reached}
        \STATE Choose $\action$ from $\state$ using a stochastic policy as defined in Eq. \eqref{ap:eq:stochargmaxPolicy} using $Q(\state, .; \theta)$.
        \STATE Take action $\textbf{a}$, observe reward $r(\state,\action)$ and next state $\state'$.
        \STATE Store $(\state, \action, r(\state,\action), \state')$ in replay buffer $\mathcal{E}$.
        \STATE Compute target values for the mini-batch:
        \[
        y_i = \begin{cases}
            r_i & \text{if } \state'_i \text{ is terminal} \\
            r_i + \gamma \hat{Q}(\state'_i, \stochargmax_{\action' \in \mathcal{A}} \hat{Q}(\state'_i, \action'; \theta^-); \theta^-) & \text{otherwise}.
        \end{cases}
        \]
        \STATE Perform a gradient descent step on the loss:
        \[
        \mathcal{L}(\theta) = \frac{1}{\lceil\log(n)\rceil} \sum_{i=1}^{\lceil\log(n)\rceil}(y_i - Q(\state_i, \action_i; \theta))^2.
        \]
        \STATE Update the target network weights:
        \[
        \theta^- \leftarrow \tau \cdot \theta + (1 - \tau) \cdot \theta^-.
        \]
        \STATE Update the Q-network weights using gradient descent:
        \[
        \theta \leftarrow \theta + \alpha \nabla_\theta \mathcal{L}(\theta).
        \]
        \STATE $\state \leftarrow \state'$.
\ENDWHILE
\ENDFOR
\end{algorithmic}
\end{algorithm}

\section{Experimental Details}
\label{appendix_exp_details}

\subsection{Environments}

We test our proposed algorithms on a standardized set of environments using open-source libraries. We compare stochastic maximization to exact maximization and evaluate the proposed stochastic RL algorithms on Gymnasium environments \citep{brockman2016openai}. Stochastic Q-learning and Stochastic Double Q-learning are tested on the CliffWalking-v0, the FrozenLake-v1, and a generated MDP environment, while stochastic deep Q-learning approaches are tested on MuJoCo control tasks \citep{todorov2012mujoco}. 

\subsubsection{Environments with Discrete States and Actions}

We generate an MDP environment with 256 actions, with rewards following a normal distribution of mean -50 and standard deviation of 50, with 3 states. Furthermore, while our approach is designed for large discrete action spaces, we tested it in Gymnasium environments \citep{brockman2016openai} with only four discrete actions, such as CliffWalking-v0 and FrozenLake-v1. CliffWalking-v0 involves navigating a grid world from the starting point to the destination without falling off a cliff. FrozenLake-v1 requires moving from the starting point to the goal without stepping into any holes on the frozen surface, which can be challenging due to the slippery nature of the ice.

\subsubsection{Environments with Continuous States: Discretizing Control Tasks}

We test the stochastic deep Q-learning approaches on MuJoCo \citep{todorov2012mujoco} for continuous states discretized control tasks. We discretize each action dimension into $i$ equally spaced values, creating a discrete action space with $n = i^d$ $d$-dimensional actions. We mainly focused on the inverted pendulum and the half-cheetah. The inverted pendulum involves a cart that can be moved left or right, intending to balance a pole on top using a 1D force, with $i=512$ resulting in 512 actions. The half-cheetah is a robot with nine body parts aiming to maximize forward speed. It can apply torque to 6 joints, resulting in 6D actions with $i=4$, which results in 4096 actions.

\subsection{Algorithms}

\subsubsection{Stochastic Maximization}

We have two scenarios, one for discrete and the other for continuous states. For discrete states, $\mathcal{E}$ is a dictionary with the keys as the states in $\sset$ with corresponding values of the latest played action in every state. In contrast, $\mathcal{E}$ comprises the actions in the replay buffer for continuous states. Indeed, we do not consider the whole set $\mathcal{E}$ either. Instead, we only consider a subset $\mathcal{M} \subset \mathcal{E}$. For discrete states, for a given state $\state$, $\mathcal{M}$ includes the latest two exploited actions in state $\state$.
For continuous states, where it is impossible to retain the last exploited action for each state, we consider randomly sampled subset $\mathcal{M} \subset \mathcal{E}$, which includes $\lceil\log(n)\rceil$ actions, even though they were played in different states. In the experiments involving continuous states, we demonstrate that this was sufficient to achieve good results, see Section \ref{Stochastic_DQN_Reward_Analysis}. 

\subsubsection{Tabular Q-learning Methods}

We set the training parameters the same for all the Q-learning variants. We follow similar hyper-parameters as in \citep{hasselt2010double}. We set the discount factor $\gamma$ to 0.95 and apply a dynamical polynomial learning rate $\alpha$ with $\alpha_t(\state, \action) = 1/z_t(\state, \action)^{0.8}$, where $z_t(\state, \action)$ is the number of times the pair $(\state, \action)$ has been visited, initially set to one for all the pairs. For the exploration rate, we use use a decaying $\varepsilon$, defined as $\varepsilon(\state) = 1/
\sqrt(z(\state))$ where $z(\state)$ is the number of times state $\state$ has been visited, initially set to one for all the states. For
Double Q-learning $z_t(\state, \action) = z^A_t(\state, \action)$ if $Q^A$ is updated and $z_t(\state, \action) = z^B_t(\state, \action)$ if $Q^B$ is updated, where $z^A_t$ and $z^B_t$ store the number of updates for each action for the corresponding value function. We averaged the results over ten repetitions. For Stochastic Q-learning, we track a dictionary $\mathcal{D}$ with keys being the states, and values being the latest exploited action. Thus, for a state $\state$, the memory $\mathcal{M} = \mathcal{D}(\state)$, thus $\mathcal{M}$ is the latest exploited action in the same state $\state$.

\subsubsection{Deep Q-network Methods}

We set the training parameters the same for all the deep Q-learning variants. We set the discount factor $\gamma$ to 0.99 and the learning rate $\alpha$ to 0.001. Our neural network takes input of a size equal to the sum of the dimensions of states and actions with a single output neuron. The network consists of two hidden linear layers, each with a size of 64, followed by a ReLU activation function \citep{nair2010rectified}. We keep the exploration rate $\varepsilon$ the same for all states, initialize it at 1, and apply a decay factor of 0.995, with a minimum threshold of 0.01. For $n$ total number of actions, during training, to train the network, we use stochastic batches of size $\lceil\log(n)\rceil$ uniformly sampled from a buffer of size $2\lceil\log(n)\rceil$. We averaged the results over five repetitions. For the stochastic methods, we consider the actions in the batch of actions as the memory set $\mathcal{M}$. We choose the batch size in this way to keep the complexity of the Stochastic Q-learning within $\mathcal{O}(\log(n))$.

\subsection{Compute and Implementation} 

We implement the different Q-learning methods using Python 3.9, Numpy 1.23.4, and Pytorch 2.0.1. For proximal policy optimization (PPO) \citep{schulman2017proximal}, asynchronous actor-critic (A2C) \cite{mnih2016asynchronous}, and deep deterministic policy gradient (DDPG) \cite{lillicrap2015continuous}, we use the implementations of Stable-Baselines \cite{stable-baselines}. We test the training time using a CPU 11th Gen Intel(R) Core(TM) i7-1165G7 @ 2.80GHz 1.69 GHz. with 16.0 GB RAM.

\newpage
\section{Additional Results}

\subsection{Wall Time Speed}
\label{appendix_experiment_time}

Stochastic maximization methods exhibit logarithmic complexity regarding the number of actions, as confirmed in Fig. \ref{fig_app:time}. Therefore, both StochDQN and StochDDQN, which apply these techniques for action selection and updates, have exponentially faster execution times compared to both DQN and DDQN, which can be seen in Fig \ref{fig_app:beta_pendulum} which shows the complete step duration for deep Q-learning methods, which include action selection and network update. The proposed methods are nearly as fast as a random algorithm, which samples and selects actions randomly and has no updates. 

\begin{figure}[t]
    \centering
    \begin{subfigure}[b]{0.37\textwidth}
        \includegraphics[width=\textwidth]{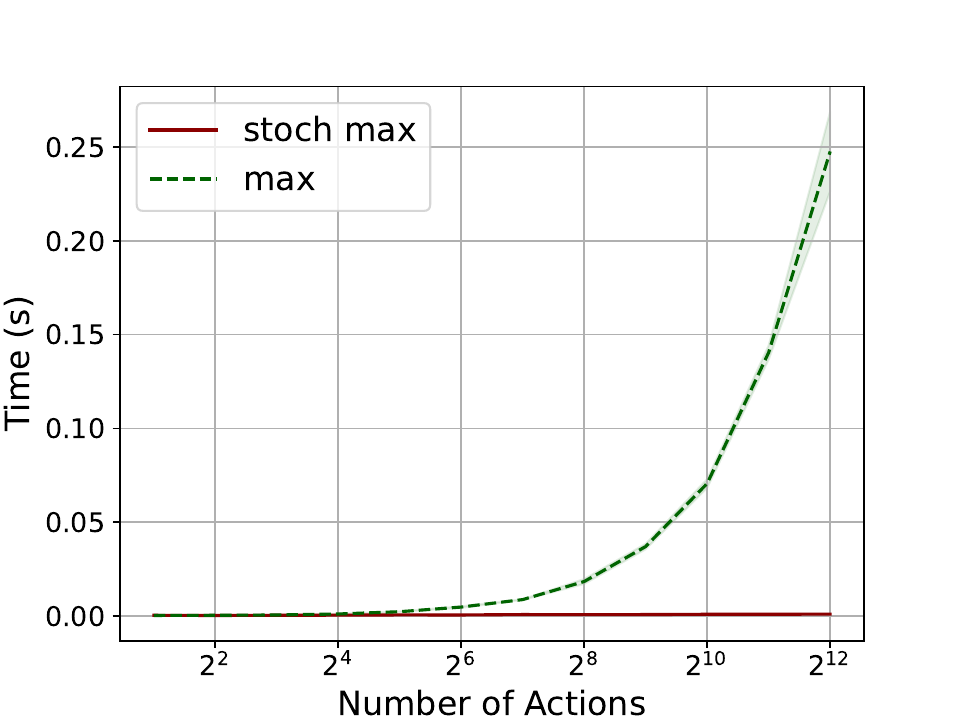}
        \caption{Action Selection Time}
        \label{fig_app:time}
    \end{subfigure}
    \hspace{0.1\textwidth}
    \begin{subfigure}[b]{0.37\textwidth}
        \includegraphics[width=\textwidth]{figures/Networkpendulum_time.pdf}
        \caption{Full Step Duration}
        \label{fig_app:beta_pendulum}
    \end{subfigure}
    \caption{Comparison results for the stochastic and deterministic methods. The x-axis represents the number of possible actions, and the y-axis represents the time step duration of the agent in seconds.}
    \label{fig_app:stoch_max_analysis_time}
\end{figure}

\subsection{Stochastic Maxmization}

\subsubsection{Stochastic Maxmization vs Maximization with Uniform Rewards} 
\label{uniform_appendx_exp}

In the setting described in Section \ref{stoch_max_with_uniform} with 5000 uniformly independently distributed action values in the range of [0, 100], as shown in Fig. \ref{reward:uniform}, $\stochmax$ without memory, i.e., $\mathcal{M}=\emptyset$ reaches around 91 in average return, and keeps fluctuating around, while $\stochmax$ with $\mathcal{M}$ quickly achieves the optimal reward.

\begin{figure}[t]
\begin{center}
\includegraphics[width=8cm]{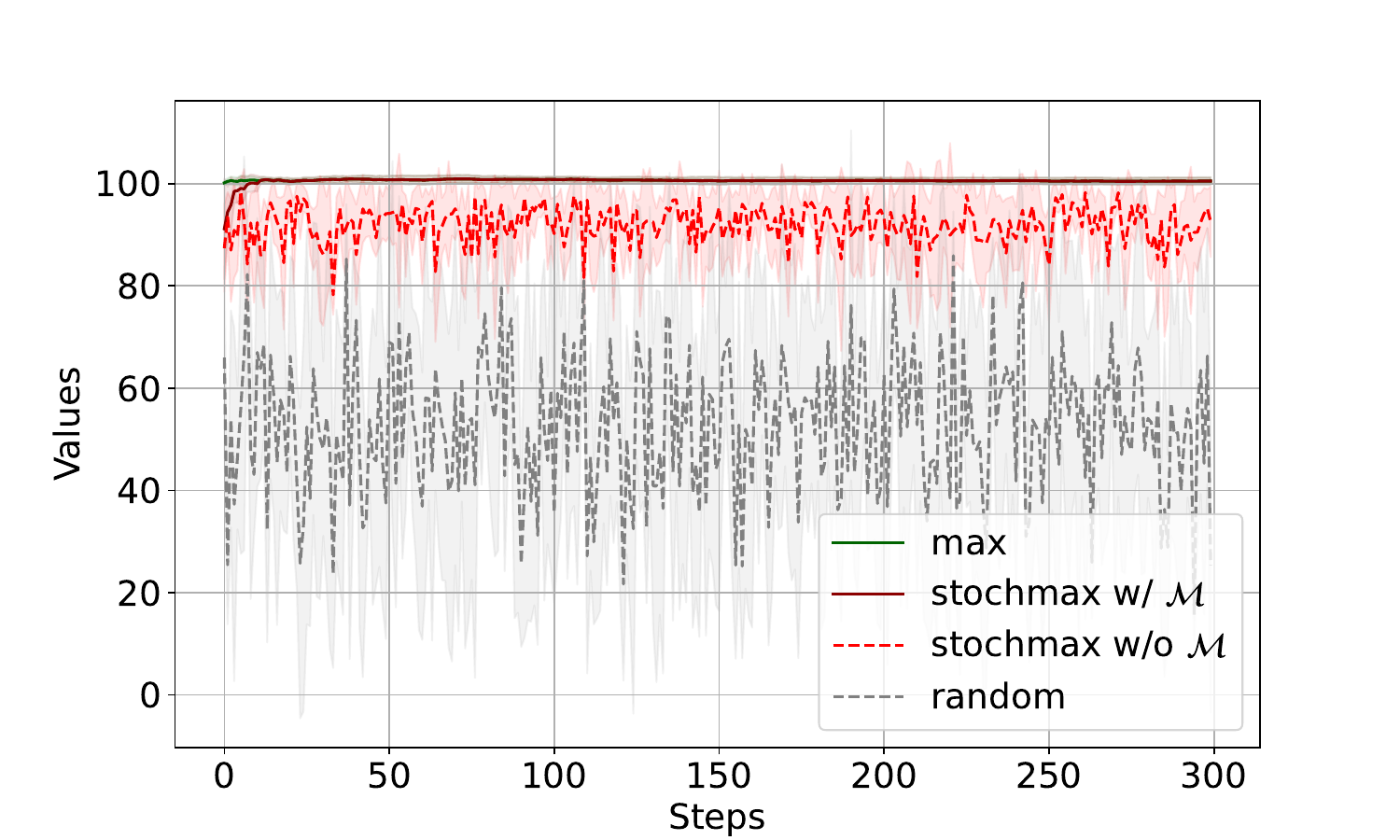}
\end{center}
\caption{$\stochmax$ with (w/) and without (w/o) memory $\mathcal{M}$ vs. $\max$ on uniformly distributed action values as described in Section \ref{sec_stoch_max_app}. The x-axis and y-axis represent the steps and the values, respectively.}
\label{reward:uniform}
\end{figure}

\subsubsection{Stochastic Maximization Analysis}
\label{appendix_experiement_stochmax}

In this section, we analyze stochastic maximization by tracking returned values across rounds, $\omega_t$ (Eq. \eqref{omega_t_def}), and $\beta_t$ (Eq. \eqref{beta_t_def}), which we provide here. At time step $t$, given a state $\state$, and the current estimated Q-function $Q_t$, we define the non-negative underestimation error as $\beta_t(\state)$, as follows:
\begin{equation}
\label{ap:beta_t_def}
\beta_t(\state) = \max _{\action \in \mathcal{A}} Q_t\left(\state, \action\right) - \stochmax _{\action \in \mathcal{A}} Q_t\left(\state, \action\right).
\end{equation}
Furthermore, we define the ratio $\omega_t(\state)$, as follows:
\begin{equation}
\label{ap:omega_t_def}
\omega_t(\state) = \frac{\stochmax _{\action \in \mathcal{A}} Q_t\left(\state, \action\right)}{ \max _{\action \in \mathcal{A}} Q_t\left(\state, \action\right)} .
\end{equation}
It follows that:
\begin{equation}
\label{ap:omega_t_def_inter}
\omega_t(\state) = 1 - \frac{\beta_t(\state)}{\max _{\action \in \mathcal{A}} Q_t\left(\state, \action\right)} .
\end{equation}

For Deep Q-Networks, for the InvertedPendulum-v4, both $\stochmax$ and max return similar values (Fig. \ref{fig_app:stochmax}), $\omega_t$ approaches one rapidly (Fig. \ref{fig_app_:omega}) and $\beta_t$ remains below 0.5 (Fig. \ref{fig_app:beta}). In the case of HalfCheetah-v4, both $\stochmax$ and max return similar values (Fig. \ref{ap_fig_cheetah:stochmax}), $\omega_t$ quickly converges to one (Fig. \ref{ap_fig_cheetah:oemga}), and $\beta_t$ is upper bounded below eight (Fig. \ref{fig_app_cheetah:beta}).

While the difference $\beta_t$ remains bounded, the values of both $\stochmax$ and max increase over the rounds as the agent explores better options. This leads to the ratio $\omega_t$ converging to one as the error becomes negligible over the rounds, as expected according to Eq. \eqref{ap:omega_t_def_inter}.

\begin{figure}[t]
    \centering
    \begin{subfigure}[b]{0.32\textwidth}
        \includegraphics[width=\textwidth]{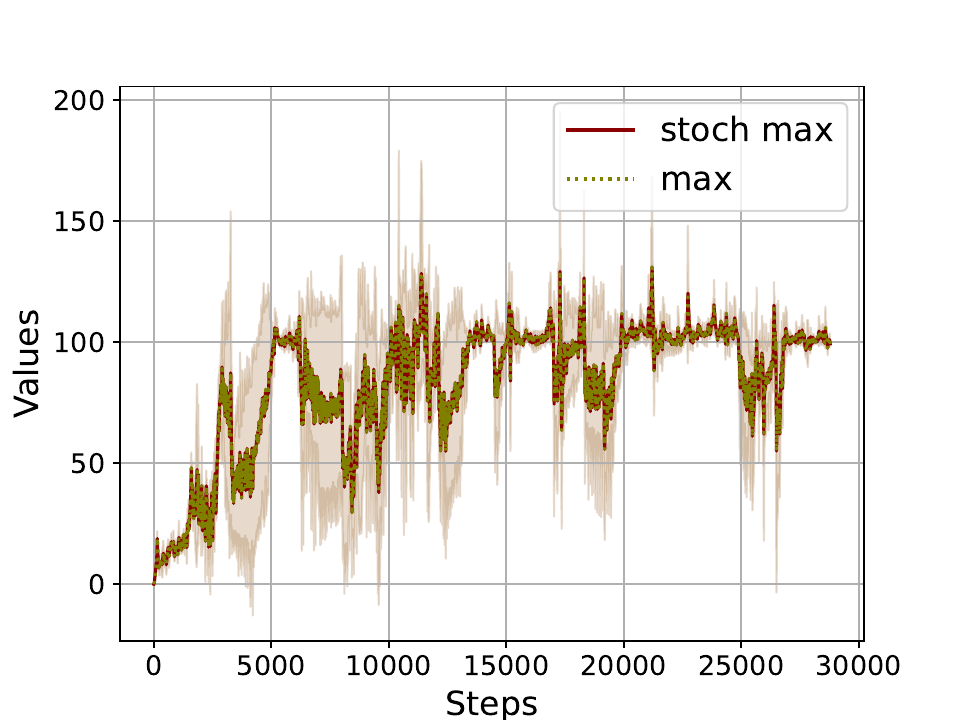}
        \caption{$\stochmax$ vs max}
        \label{fig_app:stochmax}
    \end{subfigure}
        \hfill
    \begin{subfigure}[b]{0.32\textwidth}
        \includegraphics[width=\textwidth]{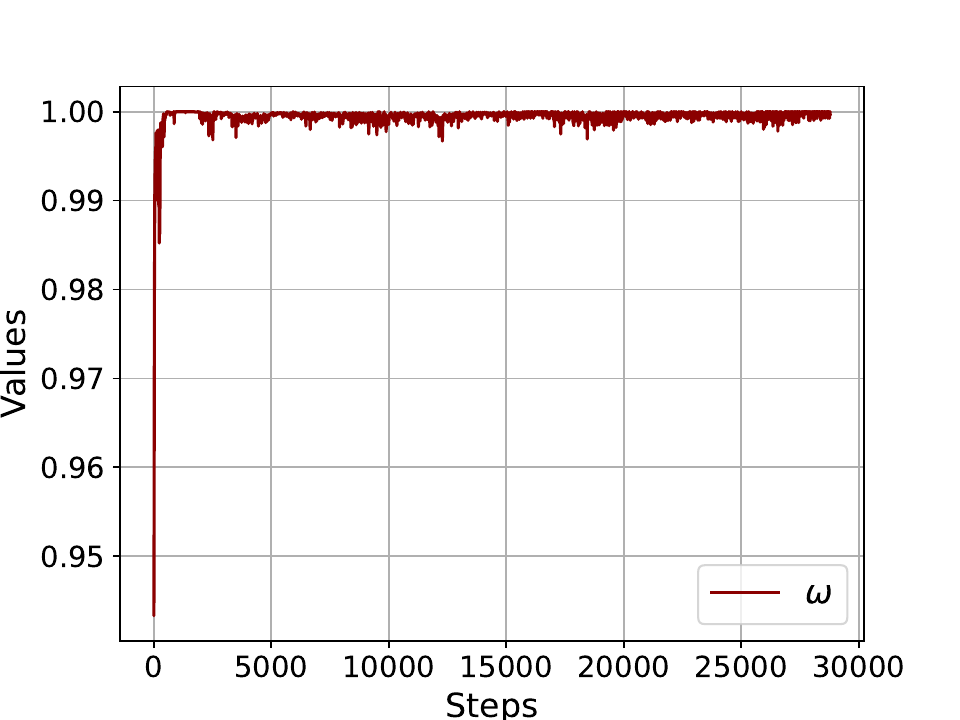}
        \caption{Ratio $\omega_t$}
        \label{fig_app_:omega}
    \end{subfigure}
    \hfill
    \begin{subfigure}[b]{0.32\textwidth}
        \includegraphics[width=\textwidth]{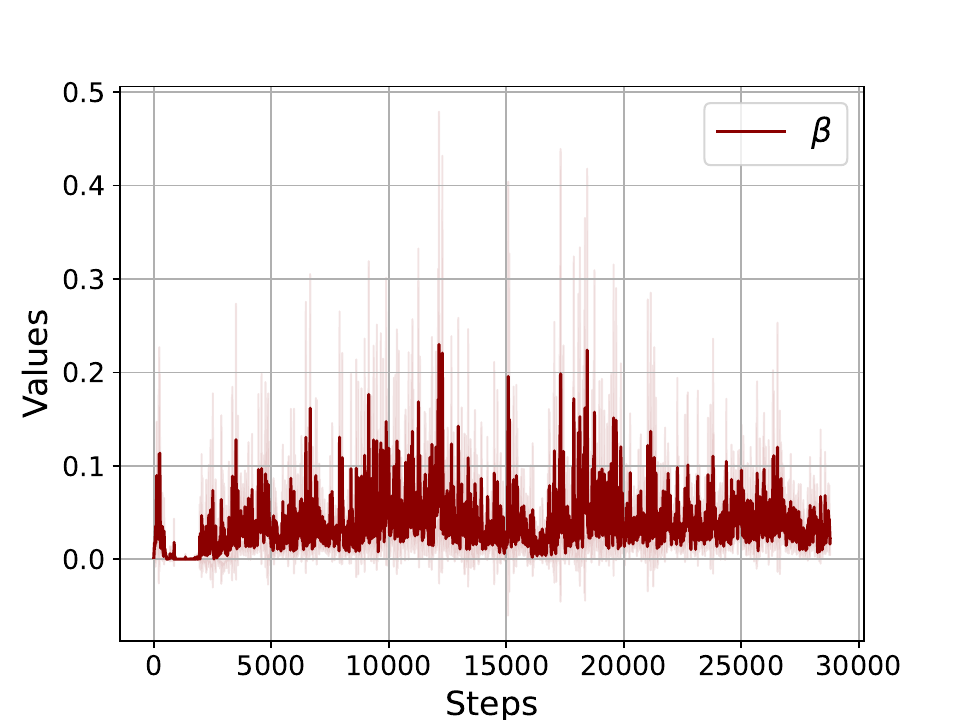}
        \caption{Difference $\beta_t$}
        \label{fig_app:beta}
    \end{subfigure}
    \caption{Comparison results for the stochastic and non-stochastic methods for the Inverted Pendulum with 512 actions. }
\end{figure}

\begin{figure}[t]
    \centering
    \begin{subfigure}[b]{0.32\textwidth}
        \includegraphics[width=\textwidth]{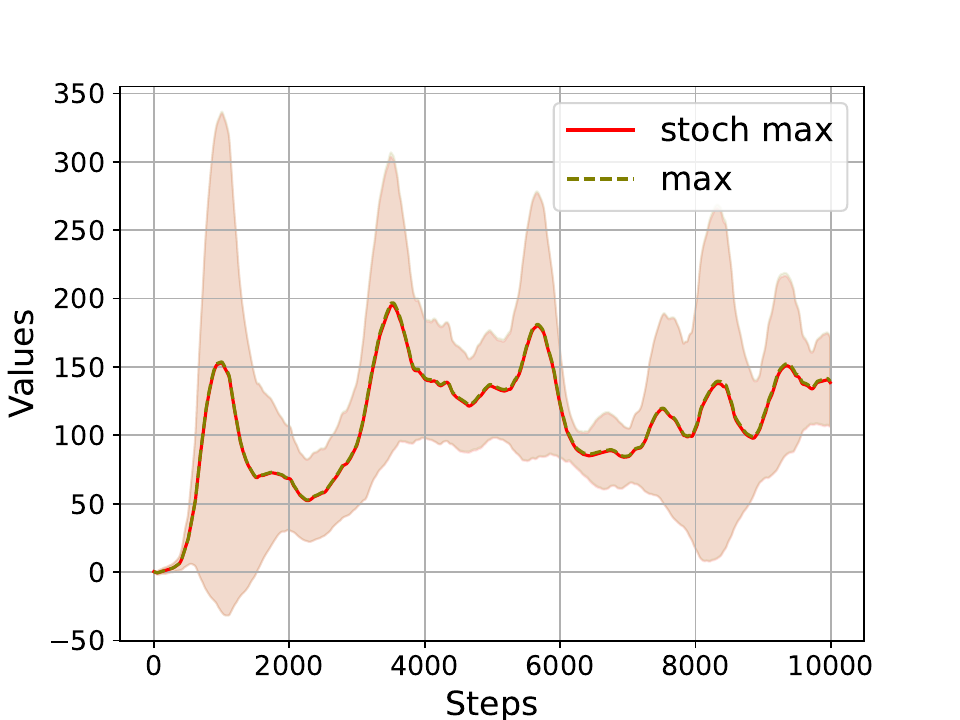}
        \caption{$\stochmax$ vs max}
        \label{ap_fig_cheetah:stochmax}
    \end{subfigure}
        \hfill
    \begin{subfigure}[b]{0.32\textwidth}
        \includegraphics[width=\textwidth]{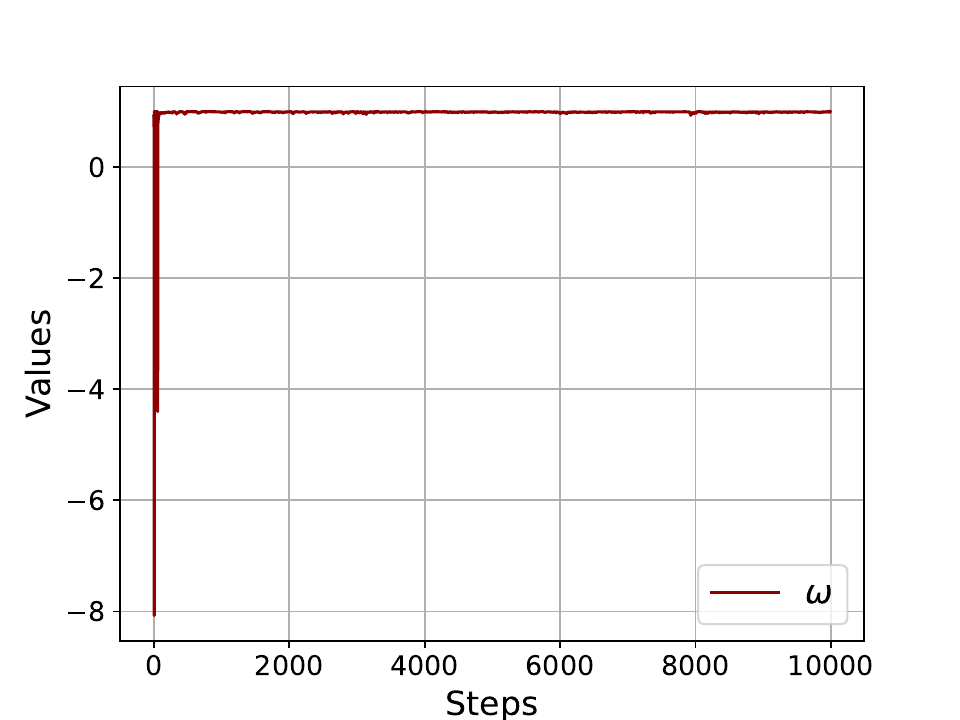}
        \caption{Ratio $\omega_t$}
        \label{ap_fig_cheetah:oemga}
    \end{subfigure}
    \hfill
    \begin{subfigure}[b]{0.32\textwidth}
        \includegraphics[width=\textwidth]{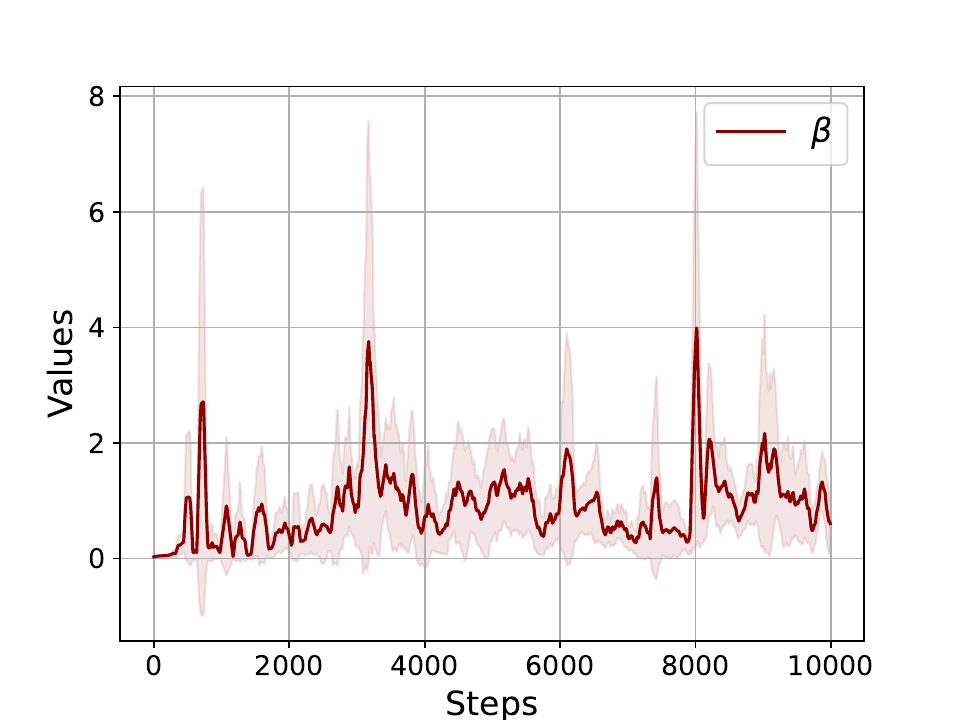}
        \caption{Difference $\beta_t$}
        \label{fig_app_cheetah:beta}
    \end{subfigure}
    \caption{Comparison results for the stochastic and non-stochastic methods for the Half Cheetah with 4096 actions. }
\end{figure}

\subsection{Stochastic Q-network Reward Analysis}
\label{appendix_experiement_dqn_rewards}

\begin{figure}[t]
    \centering
    \begin{subfigure}[b]{0.47\textwidth}
        \includegraphics[width=\textwidth]{figures/pendulum_reward_steps.pdf}
        \caption{Inverted Pendulum}
        \label{reward:ap:pendulum}
    \end{subfigure}
    \hspace{0.05\textwidth}
    \begin{subfigure}[b]{0.47\textwidth}
        \includegraphics[width=\textwidth]{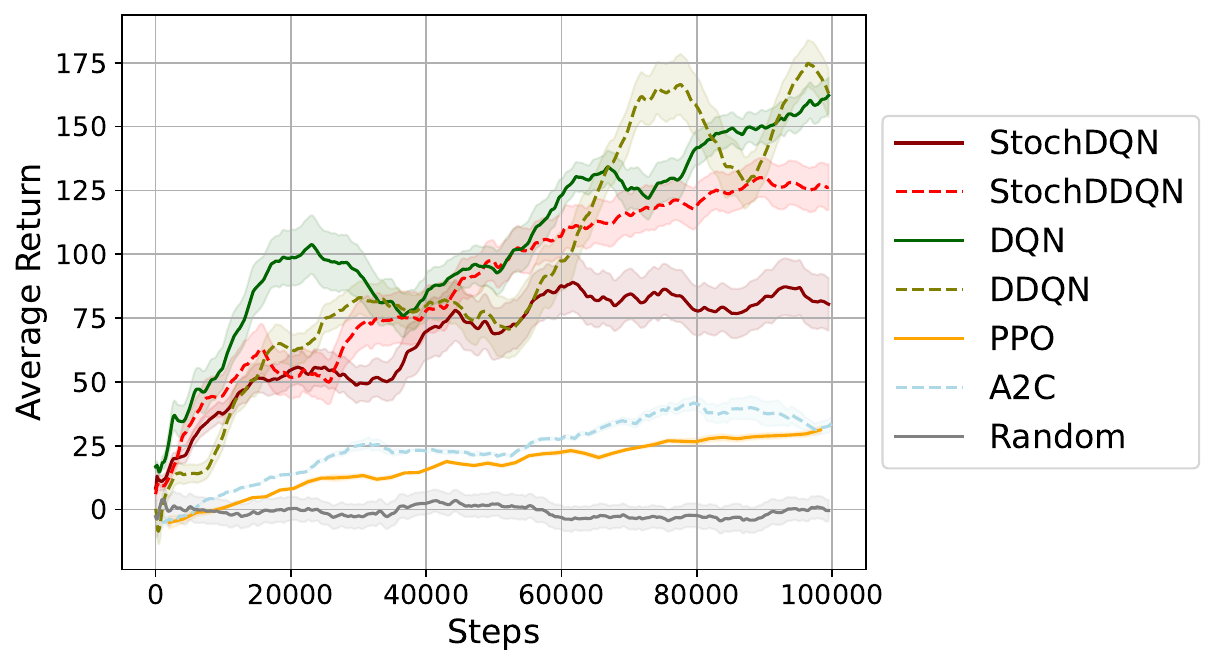}
        \caption{Half Cheetah}
        \label{reward:ap:cheetah}
    \end{subfigure}
    \caption{Stochastic vs non-stochastic of deep Q-learning variants on Inverted Pendulum and Half Cheetah, with steps on the x-axis and average returns, smoothed over a size 100 window on the y-axis.}
\end{figure}

As illustrated in Fig. \ref{reward:ap:pendulum} and Fig. \ref{reward:ap:cheetah} for the inverted pendulum and half cheetah experiments, which involve 512 and 4096 actions, respectively, both StochDQN and StochDDQN attain the optimal average return in a comparable number of rounds to DQN and DDQN. Additionally, StochDQN exhibits the quickest attainment of optimal rewards for the inverted pendulum. Furthermore, while DDQN did not perform well on the inverted pendulum task, its modification, i.e., StochDDQN, reached the optimal rewards.

\subsection{Stochastic Q-learning Reward Analysis}
\label{appendix_experiement_rewards}

We tested Stochastic Q-learning, Stochastic Double Q-learning, and Stochastic Sarsa in environments with both discrete states and actions. Interestingly, as shown in Fig. \ref{ap:fig:frozen_lake}, our stochastic algorithms outperform their deterministic counterparts in terms of cumulative rewards. Furthermore, we notice that Stochastic Q-learning outperforms all the considered methods regarding the cumulative rewards. Moreover, in the CliffWalking-v0 (as shown in Fig. \ref{ap:fig:clif_walking}), as well as for the generated MDP environment with 256 possible actions (as shown in Fig. \ref{ap:fig:mdp}), all the stochastic and non-stochastic algorithms reach the optimal policy in a similar number of steps. 


\begin{figure}[t]
    \centering
    \begin{subfigure}[b]{0.55\textwidth}
        \includegraphics[width=\textwidth]{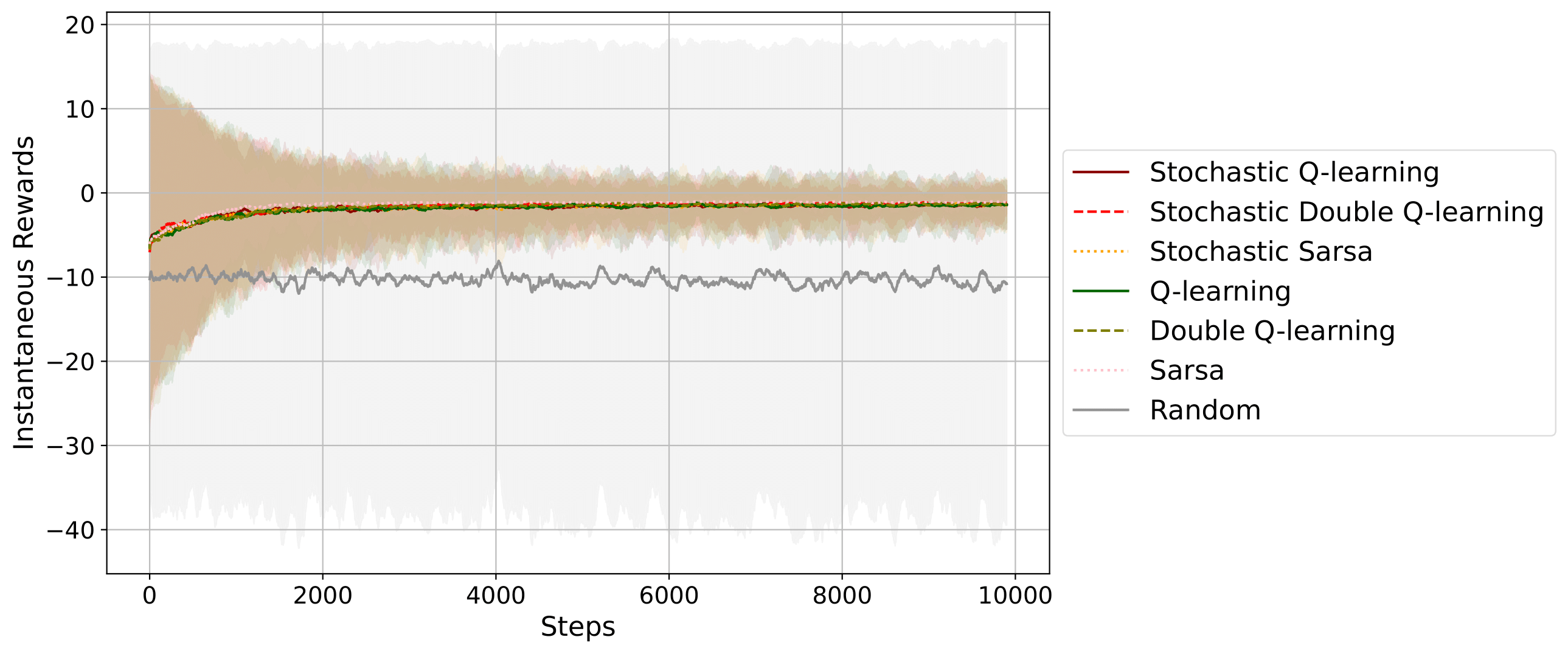}
        \caption{Instantaneous Rewards}
    \end{subfigure}
        \hspace{0.05\textwidth}
    \begin{subfigure}[b]{0.39\textwidth}
        \includegraphics[width=\textwidth]{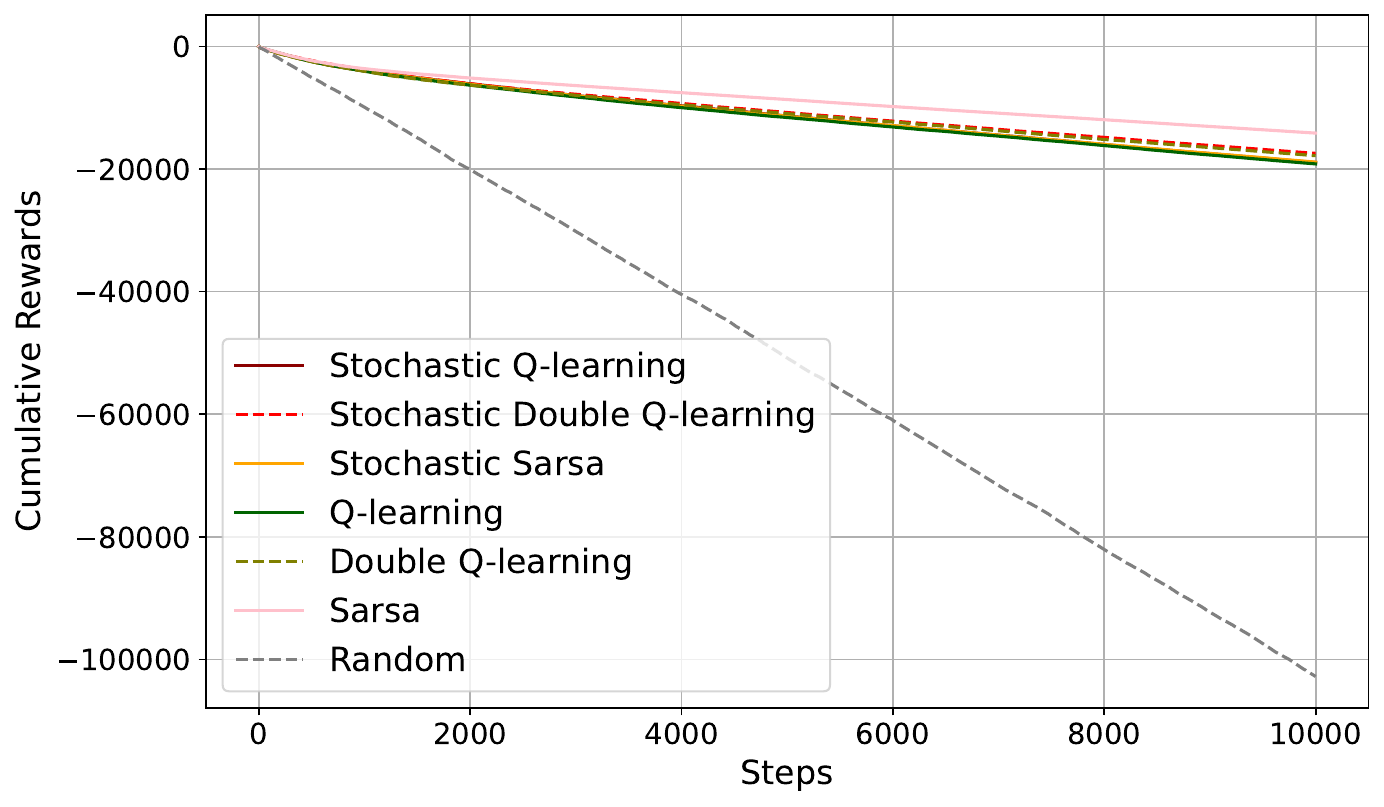}
        \caption{Cumulative Rewards}
    \end{subfigure}
    \caption{Comparing stochastic and non-stochastic Q-learning approaches on the Cliff Walking, with steps on the x-axis, instantaneous rewards smoothed over a size 1000 window on the y-axis for plot (a), and cumulative rewards on the y-axis for plot (b).}
    \label{ap:fig:clif_walking}
\end{figure}

\begin{figure}[t]
    \centering
    \begin{subfigure}[b]{0.55\textwidth}
        \includegraphics[width=\textwidth]{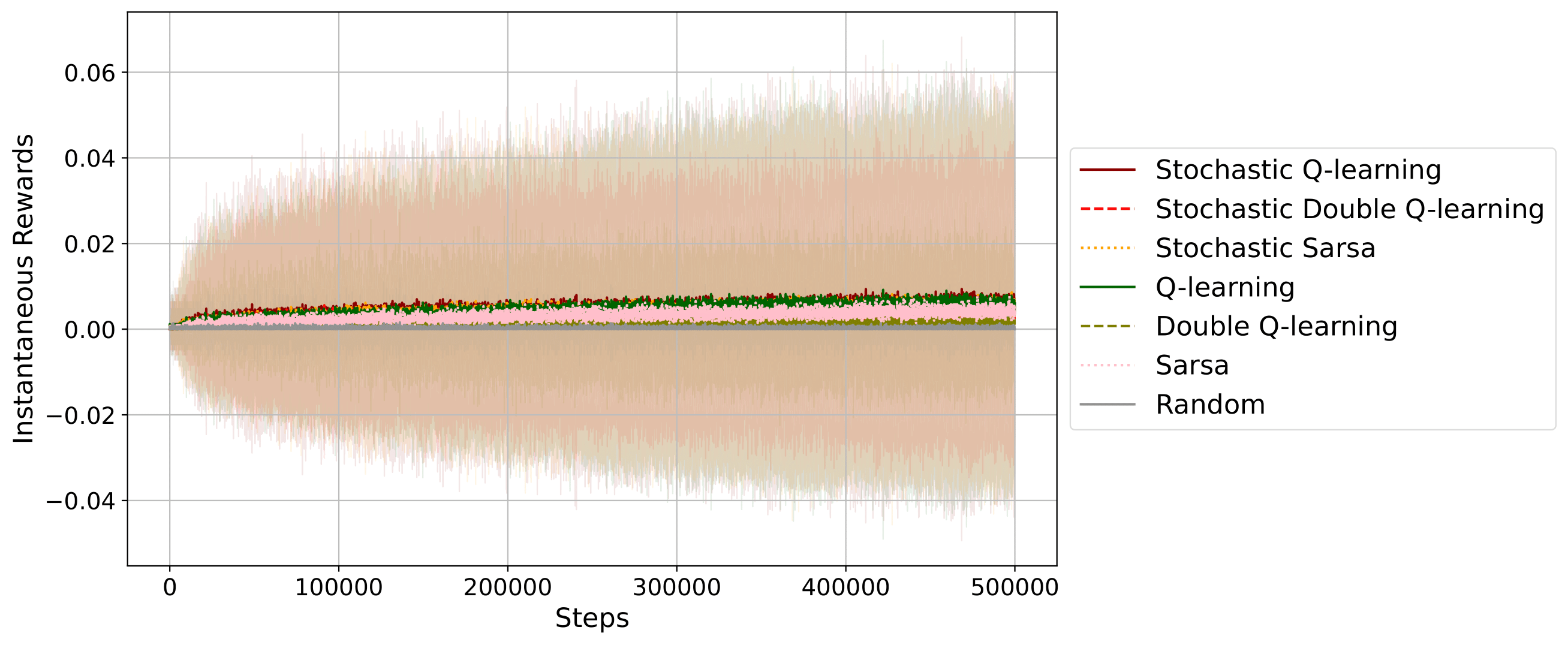}
        \caption{Instantaneous Rewards}
    \end{subfigure}
        \hspace{0.05\textwidth}
    \begin{subfigure}[b]{0.39\textwidth}
        \includegraphics[width=\textwidth]{figures/cumulative_rewards_frozen.pdf}
        \caption{Cumulative Rewards}
    \end{subfigure}
    \caption{Comparing stochastic and non-stochastic Q-learning approaches on the Frozen Lake, with steps on the x-axis, instantaneous rewards smoothed over a size 1000 window on the y-axis for plot (a), and cumulative rewards on the y-axis for plot (b).}
    \label{ap:fig:frozen_lake}
\end{figure}

\begin{figure}[t]
    \centering
    \begin{subfigure}[b]{0.55\textwidth}
        \includegraphics[width=\textwidth]{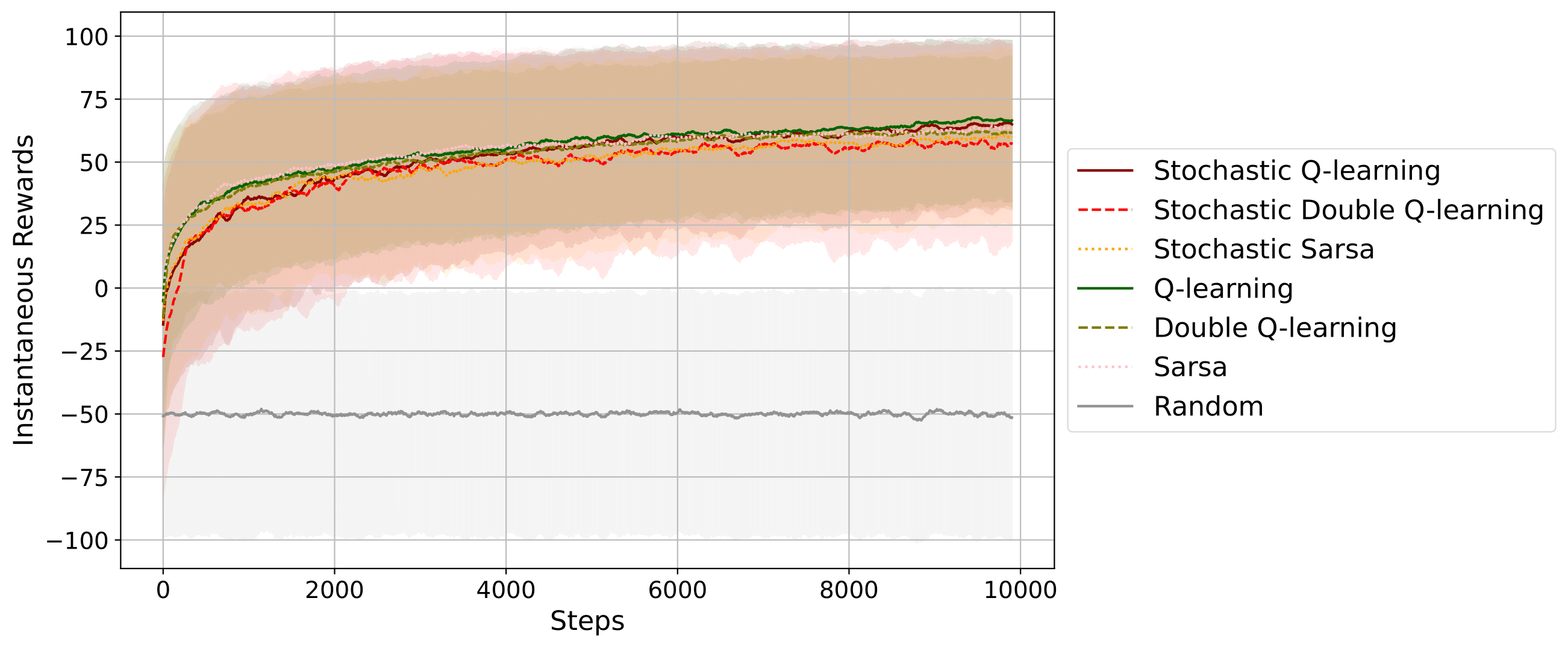}
        \caption{Instantaneous Rewards}
    \end{subfigure}
        \hspace{0.05\textwidth}
    \begin{subfigure}[b]{0.39\textwidth}
        \includegraphics[width=\textwidth]{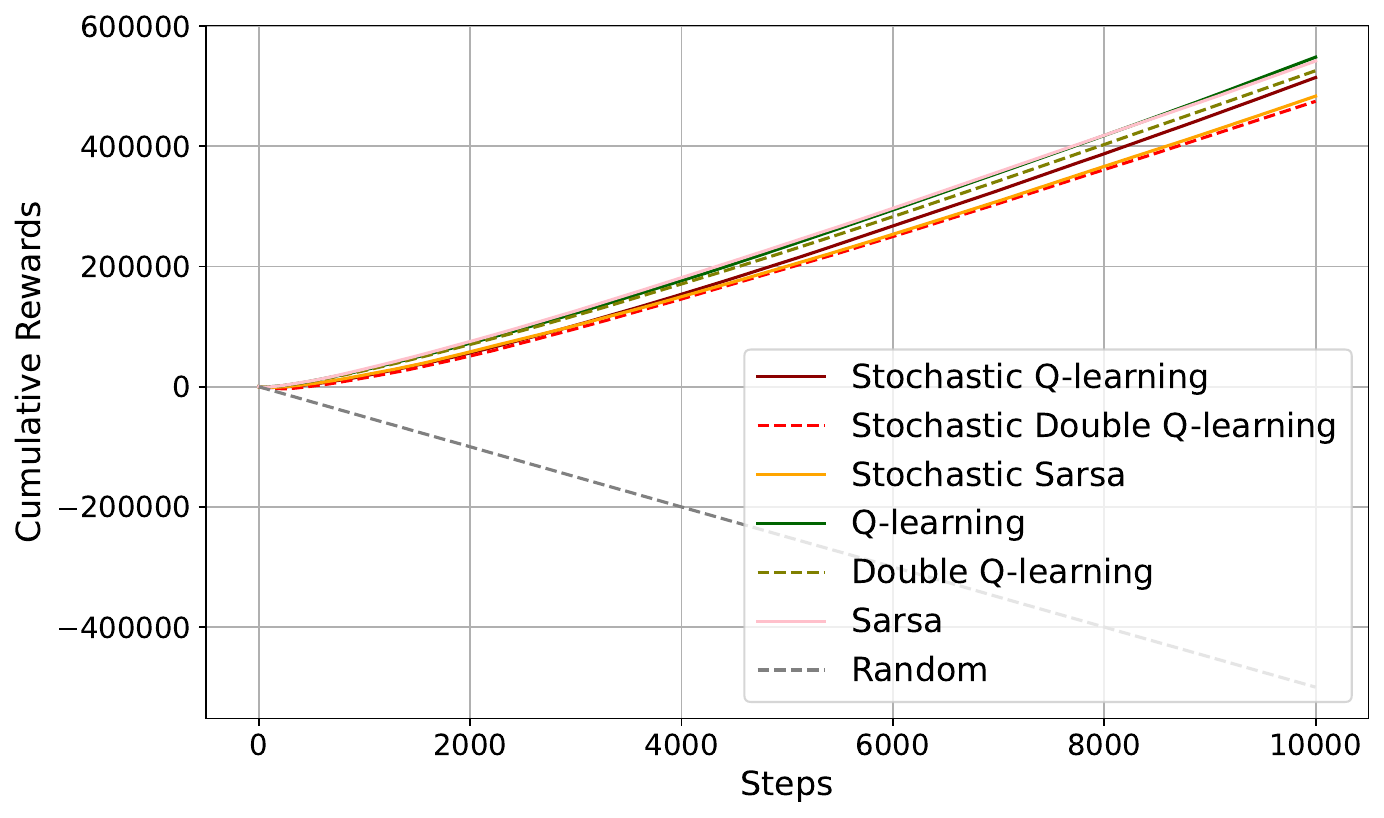}
        \caption{Cumulative Rewards}
    \end{subfigure}
    \caption{Comparing stochastic and non-stochastic Q-learning approaches on the generated MDP environment, with steps on the x-axis, instantaneous rewards smoothed over a size 1000 window on the y-axis for plot (a), and cumulative rewards on the y-axis for plot (b).}
    \label{ap:fig:mdp}
\end{figure}

\newpage

\end{document}